%% file: main.tex
\documentclass{article}

\PassOptionsToPackage{numbers, compress}{natbib}

\usepackage[preprint]{neurips_2023}

\usepackage[utf8]{inputenc} 
\usepackage[T1]{fontenc}    
\usepackage{hyperref}       
\usepackage{url}            
\usepackage{booktabs}       
\usepackage{amsfonts}       
\usepackage{nicefrac}       
\usepackage{microtype}      
\usepackage{xcolor}         

\title{Perceptual Kalman Filters: Online State Estimation under a Perfect Perceptual-Quality Constraint}

\author{%
  Dror Freirich\\
 Technion - Israel Institute of Technology \\
 \texttt{drorfrc@gmail.com}
  \And
  Tomer Michaeli\\
  Technion - Israel Institute of Technology \\
  \texttt{tomer.m@ee.technion.ac.il} \\
   \And
 Ron Meir\\
 Technion - Israel Institute of Technology \\
  \texttt{rmeir@ee.technion.ac.il} \\
}

\input{mydefinitions}

\begin{document}

\maketitle

\begin{abstract}
Many practical settings call for the reconstruction of temporal signals from corrupted or missing data. Classic examples include decoding, tracking, signal enhancement and denoising. Since the reconstructed signals are ultimately viewed by humans, it is desirable to achieve reconstructions that are pleasing to human perception.
Mathematically, perfect perceptual-quality is achieved when the distribution of restored signals is the same as that of natural signals, a requirement which has been heavily researched in static estimation settings (i.e.~when a whole signal is processed at once). 
Here, we study the problem of optimal \emph{causal}  filtering under a perfect perceptual-quality constraint, which is a task of fundamentally different nature.  
Specifically, we analyze a Gaussian Markov signal observed through a linear noisy transformation. In the absence of perceptual constraints, the Kalman filter is known to be optimal in the MSE sense for this setting. Here, we show that adding the perfect perceptual quality constraint (i.e.~the requirement of temporal consistency), introduces a fundamental dilemma whereby the filter may have to ``knowingly'' ignore new information revealed by the observations in order to conform to its past decisions. This often comes at the cost of a significant increase in the MSE (beyond that encountered in static settings). Our analysis goes beyond the classic innovation process of the Kalman filter, and introduces the novel concept of an unutilized information process. Using this tool, we present a recursive formula for perceptual filters, 
and demonstrate the qualitative effects
of perfect perceptual-quality estimation on a video reconstruction problem.
\end{abstract}

\input{pkl}
\input{pk_demo_short}

\bibliography{dpreferences,pkalreferences}
\bibliographystyle{plainnat}

\newpage
\appendix
\input{pkl_appendix.tex}

\end{document}

%% file: pkl.tex
\section{Introduction}

In many settings, it is desired to reconstruct a temporal signal from corrupted or missing data. Examples include decoding of transmitted communications, tracking targets based on noisy measurements, enhancing audio signals, and denoising videos.  
Traditionally, restoration quality has been assessed by distortion
measures such as MSE. As a result, numerous methods targeted the minimization of such measures, including
the seminal work of \citet{kalman1960filtering}. However, in applications involving human perception,
one may favor reconstructions that cannot be told apart from valid signals. Mathematically, such \textit{perfect perceptual
quality} can be achieved only if the distribution of restored signals is the same as that of ``natural'' signals.

Interestingly, it has been shown that good {perceptual
quality} generally comes at the price of poor distortion and vice versa. This phenomenon, known as the \textit{perception-distortion tradeoff}, was first studied in \citep{blau2018perception}, and was later fully characterized in \citep{freirich2021theory} for the particular setting where distortion is measured by MSE and perception is measured by the Wasserstein-$2$ distance between the distribution of estimated signals and the distribution of real signals.  However, to date, all existing works addressed the static (non-temporal) setting, in which the entire corrupted signal is available for processing all at once. This setting is fundamentally different from situations involving temporal signals, in which the corrupted signal is processed causally over time, such that each sample is reconstructed only based on observations up to the current time.

\begin{wrapfigure}{R}{0.5\linewidth}
\center \includegraphics[viewport=60bp 100bp 870bp 450bp,clip,width=.99\linewidth]{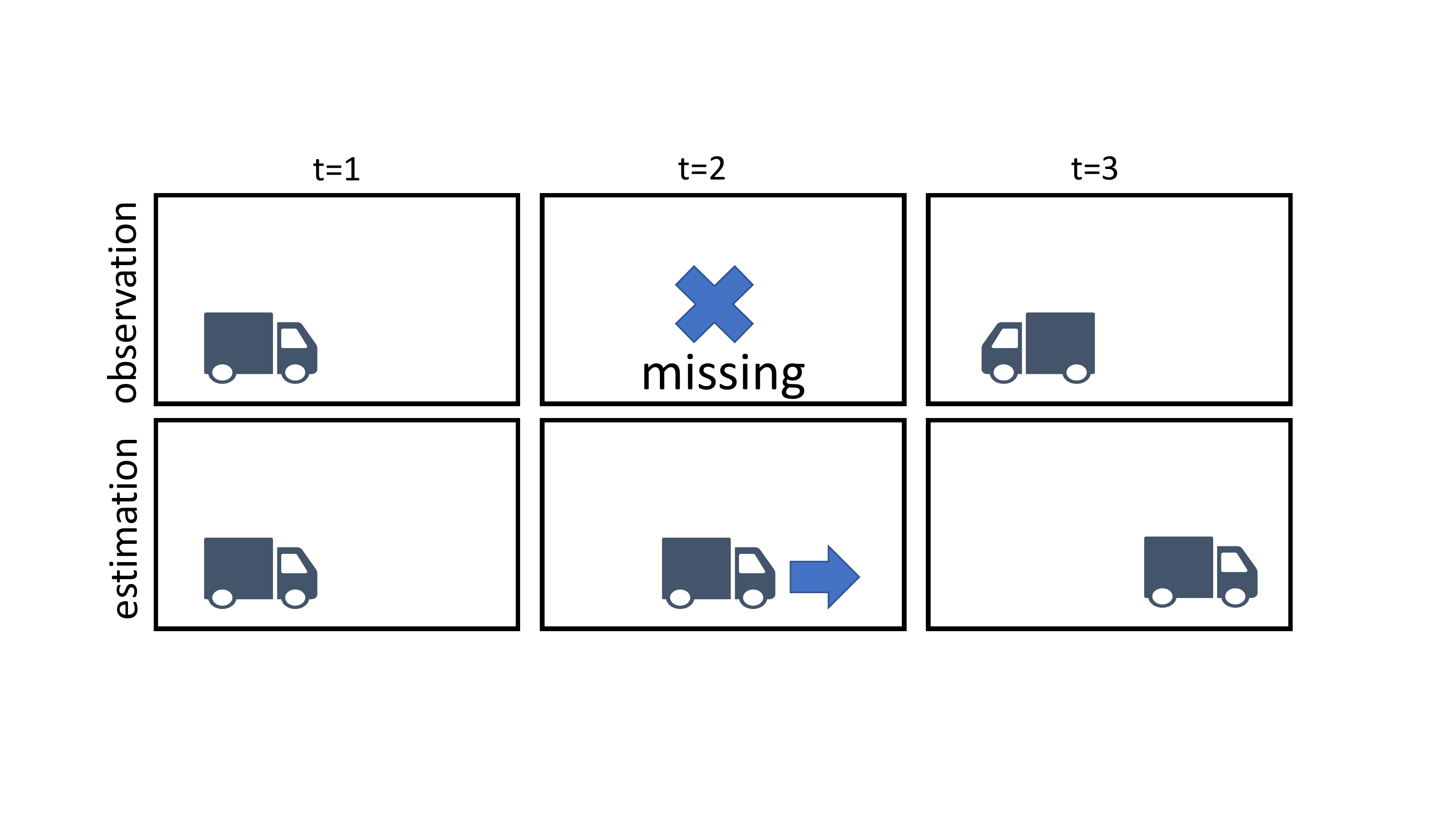}
\caption[The temporal consistency dilemma.]{\textbf{The temporal consistency dilemma.} \label{fig:dilemma}
Estimation cannot suddenly change the motion in the output video, because such an abrupt change would deviate from natural video statistics. Thus, although the method is aware of its mistake, it may have to stick to its past decisions.}
\end{wrapfigure}

To illustrate the inherent difficulty in causal estimation, consider video restoration tasks like denoising, super-resolution, or frame completion (see Fig.~\ref{fig:dilemma}). Achieving high perceptual quality in those tasks requires generating restored videos whose spatio-temporal distribution matches that of natural videos. Particularly, an incorrect temporal distribution may lead to flickering artifacts~\citep{perez2018perceptual} or to unnaturally slow dynamics~\citep{chu2020learning}. To comply with this requirement, the restoration method needs to `hallucinate' motion whenever the dynamics cannot be accurately determined from the measurements. For example, it may be impossible to determine whether a car is standing still or moving slowly from just a few noisy frames, yet the restoration method must generate \emph{some} (nearly) constant velocity in order to comply with the statistics of natural videos. However, as more measurements become available, the uncertainty may be reduced, and it may become evident that the hallucinated dynamics were in fact incorrect. When this happens, the method cannot suddenly change the motion in the output video, because such an abrupt change would deviate from natural video statistics. Thus, although the method ``becomes aware'' of its mistake, it may have to stick to its past decisions for a while. A natural question is, therefore:
\begin{center} 
    \textit{What is the precise cost of temporal consistency in online restoration?}
\end{center} 

In this paper, we study this question in the setting where the signal to be restored, $x_t$, is a discrete-time Gaussian Markov  process, and the measurements $y_t$ are noisy linear transformations of the signal's state. 
We address the problem of designing a \textit{causal filter} for estimating $x_t$ from $y_t$, where the distribution law of the filter's output, $\hat{x}_t$, is constrained to be the same as that of $x_t$ (perfect perceptual quality).
We show that this temporal consistency constraint indeed comes at the cost of increased MSE compared to filters that only enforce the correct distribution per time step, but not joint distributions across time steps. To derive a recursive form for linear perceptual filters, we introduce the novel concept of an \textit{unutilized information} process, which is the portion of accumulated information that does not depend on past estimates. We provide a closed-form expression for the MSE of such filters and show how to design their coefficients to minimize different objectives. We further establish a special class of perceptual
filters, based on the classic \textit{innovation} process, which has an explicit solution.  
We analyze the evolution of MSE over time for perceptual filters and for non-perceptual ones in several numerical setups. Finally, we demonstrate the qualitative effects of perfect perceptual-quality estimation on a simplified video reconstruction problem.

\paragraph{Related work}
Many works proposed practical algorithms for achieving high (spatio-temporal) perceptual quality in video restoration tasks.%
\citet{bhattacharjee2017temporal}
improved temporal coherence by using a loss that penalizes discontinuities between sequential frames. 
\citet{perez2018perceptual}
suggested a recurrent generator architecture whose inputs include the low-resolution current frame (at time $t$), the high-resolution reconstruction of the previous one (at time $t-1$), and a low-resolution version of the previous frame, aligned to the current one. The model is trained using losses that encourage it to conserve the joint statistics
between consecutive frames.
\citet{chu2020learning} introduced a temporally coherent GAN
architecture (TecoGAN). Their generator's input includes again a warped version of the previously generated frame, where each discriminator input consists of 3 consecutive frames, either generated or ground-truth. They also introduced a bi-directional loss that encourages long-term consistency by avoiding temporal accumulation of artifacts, and another loss which measures the similarity between motions. 
More  recent progress in generating temporally coherent videos includes \textit{Make-a-video} \citep{singer2022make}  that expands a text-to image model with spatio-temporal convolutional and attention layers, and Video-Diffusion Models \citep{ho2022video} that use a diffusion model with a network architecture adapted to video.
In the context of online restoration, we mention the work of \citet{Kim_2018_ECCV} which presented a GAN architecture for real-time video deblurring, where restoration is done sequentially. They introduce a network layer for dynamically (at test-time) blending  features between consecutive frames. This mechanism enables generated features to propagate into future frames, thus improving consistency. We note, however, that our work is the first to provide a theoretical/mathematical framework, and a closed-form solution for a special case.

\section{Preliminaries: The distortion-perception tradeoff}
\label{sec::prelim:DP}
Let $x,y$ be random vectors taking values in $\mathbb{R}^{n_{x}}$
and $\mathbb{R}^{n_{y}}$, respectively, with joint probability $p_{xy}$. Suppose we want to estimate $x$ based on $y$, such that the estimator $\hat{x}$ satisfies two requirements: (i) It has a low distortion $\EE[d(x,\hat{x})]$, where $d(\cdot,\cdot)$ is some measure of discrepancy between signals; (ii) It has a good perceptual quality, \textit{i.e.}~it achieves a low value of $d_{p}(p_{x},p_{\hat{x}})$, where $d_{p}(\cdot,\cdot)$ is a divergence
between probability measures.
\citet{blau2018perception} studied the best possible distortion that can be achieved under a given level of perceptual quality, by introducing the \textit{distortion-perception} function 
\begin{equation}
D(P)=\min_{p_{\hat{x}|y}}\left\{ \EE[d(x,\hat{x})]\;:\;d_{p}(p_{x},p_{\hat{x}})\leq P\right\}. \label{eq:D_P::General_definition}
\end{equation}
\citet{freirich2021theory} provided a complete characterization of $D(P)$ for the case where $d$ is the squared-error and $d_p$ is the Wasserstein-$2$ distance. Particularly, in the Gaussian case, they developed a closed-form expression for the optimal estimator. 

In this paper we discuss estimation with perfect perceptual quality, namely $P = 0$. 
In this case, \citep[Thm.~4]{freirich2021theory}
implies that if $x$ and $y$ are zero-mean, jointly-Gaussian with covariances $\Sigma_x,\Sigma_{y} \succ0$, and
$x^*=\EEb{x|y}$,
then a MSE-optimal perfect perceptual-quality estimator is obtained by 
\begin{equation} 
\hat{x}=\TT^*x^*+w, \, 
\label{eq::Gaussians:optimalT}
\quad {\TT}^{*}\triangleq \Sigma_{x}^{\hlf}(\Sigma_{x}^{\hlf}\Sigma_{x^*}\Sigma_{x}^{\hlf})^{\hlf}\Sigma_{x}^{-\hlf}\Sigma_{x^*}^{\dagger},
\end{equation}
where $w$ is a zero-mean Gaussian noise with covariance
$\Sigma_{w}=\Sigma_{x}-\TT^{*}\Sigma_{x^*}^{}\TT^{*\T}$,
independent of $y$ and $x$, and $\Sigma_{x^*}^{\dagger}$ is the Moore-Penrose inverse of $\Sigma_{x^*}$.
For the more general case where $\Sigma_x \succeq 0$, a similar result can be obtained by using Theorem~\ref{thm:=00005BOlkin-1982,-Thm.4=00005D} in the Appendix.

\section{Problem formulation}
\label{sec::problem setting}

We consider a state
$x_{k}\in\mathbb{R}^{n_{x}}$ with linear dynamics driven by Gaussian noise, and observations $y_k \in \R^{n_y}$ that are linear transformations of $x_{k}$ perturbed by  Gaussian noise,
\begin{align}
x_{k}&=A_{k}x_{k-1}+q_{k},& k=1,...,T, \label{eq:kalman:setting1} \\
y_{k}&=C_{k}x_{k}+r_{k},& k=0,...,T.\label{eq:eq:kalman:setting2}
\end{align}
Here, the noise vectors $q_{k}\sim\mathcal{N}(0,Q_{k})$ and $r_{k}\sim\mathcal{N}(0,R_{k})$ are independent white Gaussian
processes, 
and $x_{0}\sim\mathcal{N}(0,P_0)$ is independent of $q_{1},r_{0}$. For convenience, we will sometimes refer to $P_{0}$ as $Q_{0}$. The matrices $A_{k}$, $C_{k}$, $Q_{k}$, $R_{k}$
and $P_{0}$ are deterministic system parameters with appropriate dimensions, and
assumed to be known. 

Our goal is to construct an estimated sequence $\hat{X}_{0}^{T} = (\hat{x}_{0},\ldots,\hat{x}_{T})$ based on the measurements $Y_0^T=(y_{0},\ldots,y_{T})$, which minimizes the cost 
\begin{align}
\C(\hat{x}_{0},\ldots,\hat{x}_{T}) &= 
\sum_{k=0}^{T}\w_{k}\EEb{\|{x}_{k}-\hat{x}_{k}\|^{2}}, \label{eq::cost_C:def}
\end{align}
for some given weights $\w_k\geq0$. Importantly, we want to do so under the following two constraints.
\begin{align}
    &\mathrm{Temporal~causality:~}\qquad\qquad\quad  \hat{x}_{k}\sim p_{\hat{x}_{k}}(\cdot|y_{0},\ldots,y_{k},\hat{x}_{0},\ldots,\hat{x}_{k-1}), \label{eq:causality} \\
    &\mathrm{Perfect~perceptual~quality:~} \qquad p_{\hat{X}_{0}^{T}}=p_{X_{0}^{T}}.\label{eq:perfect_perception}
\end{align}

Note that Condition \eqref{eq:perfect_perception} requires not only that every estimated sample have the same distribution as the original
one, but also that the \emph{joint} distribution of every subset of reconstructed samples be identical to that of the corresponding subset of samples in the original sequence.
In the context of video processing, this means that not only does every recovered frame have to look natural, but also that motion must look natural. This perfect perceptual quality constraint is what sets our problem apart from the classical Kalman filtering problem, which considers only the causality constraint. Since we will make use of the Kalman filter, let us briefly summarize it.

\paragraph{The Kalman filter (no perceptual quality constraint)}
Let 
$\hat{x}^*_{k|s}\triangleq
\EEb{x_k|y_{0},\dots,y_{s}}$
denote the estimator of $x_k$ based on all observations up to time $s$, which minimizes the MSE. 
The celebrated Kalman filter \citep{kalman1960filtering} is an efficient method for calculating 
the \textit{Kalman optimal state} 
$\hat{x}_{k}^{*}\equiv\hat{x}^*_{k|k}$ recursively without having to store all observations up to time $k$. It is given by the recurrence 
\begin{equation}
\hat{x}_{k}^{*}=A_{k}\hat{x}_{k-1}^{*}+K_{k}\II_{k},
\label{eq::kalman:recurrence}
\end{equation}
where $K_{k}$ is the \textit{optimal Kalman gain} \citep{kalman1960filtering}, whose recursive calculation is given in Algorithm~\ref{alg:Kalman-Filter} in the Appendix.
The vector $\II_{k}$ is the \textit{innovation} process,
\begin{equation}
\II_{k}=y_{k}-C_{k}\hat{x}^*_{k|k-1},
\end{equation}
describing the new information carried by the observation $y_k$ over the optimal prediction based on the observations up to time $k-1$, which is given by $\hat{x}^*_{k|k-1}=A_k\hat{x}_{k-1}^*$.
The innovation $\II_{k}$ is uncorrelated with all observations up to time $k-1$,
which guaranties the MSE optimality of the estimation. The calculation
of the Kalman state is also summarized in Alg.~\ref{alg:Kalman-Filter}.
Pay attention to the innovation process $\II_{k}$, its covariance $S_{k}$
and gain $K_{k}$, which we will build upon. Our notations are summarized in Table~\ref{tab:pkl:Definitions-and-Notations} in the Appendix. 
Note that since the Kalman filter minimizes the MSE at each timestep, it also minimizes \eqref{eq::cost_C:def} regardless of the choice of $\w_k$, but it generally fails to fulfill \eqref{eq:perfect_perception}. As we will see later, taking \eqref{eq:perfect_perception} into consideration, the choice of a specific optimization objective does affect the optimal filter's identity.

\paragraph{Temporally-inconsistent perceptual filter}
A naive way to try to improve the perceptual quality of the Kalman filter would be to require that each $\hat{x}_k$ be distributed like $x_k$ (but without constraining the joint distribution of samples). In the context of video processing, each frame generated by such a filter would look natural, but motion would not necessarily look natural. This problem can be solved using the result \eqref{eq::Gaussians:optimalT}, which gives the optimal ``temporally-inconsistent'' perceptual estimator
\begin{equation}
\hat{x}^{\ntc}_{k}=\TT^*_{k}\hat{x}_{k}^{*}+w_{k}=\TT^*_{k}\left(A_{k}\hat{x}_{k-1}^{*}+K_{k}\II_{k}\right)+w_{k},
\label{eq::pkl:ntc-estimator}
\end{equation}
with $\TT^*_{k}$ and $w_k$ from \eqref{eq::Gaussians:optimalT}. These quantities depend only on the covariances of $x_k,\hat x^*_k$, which can be computed recursively using the Kalman method. The MSE of this estimator can be computed in closed-form \citep{freirich2021theory} as 
\begin{align}
\EEb{\|x_k - \hat x_k\|^2 } = 
d^*_k + \tr{\Sigma_{x_k} + \Sigma_{\hat x^*_k} - 2\left(\Sigma_{x_k}^{\frac{1}{2}}\Sigma_{\hat x^*_k}\Sigma_{x_k}^{\frac{1}{2}}\right)^{\frac{1}{2}} }, 
\end{align}
where $d^*_k$ is the MSE of the Kalman filter, which can also be computed recursively.

\paragraph{Our setting (with the perceptual quality constraint)}

Going back to our setting, 
one may readily recognize that perceptually reconstructing the signal $X_{0}^{T}$ from the full measurement sequence $Y_{0}^{T}$ is also a special case of the Gaussian perceptual restoration problem discussed in 
Section~\ref{sec::prelim:DP}, only applied to the entire sequence of states and measurements. Generally, this estimate already achieves a higher MSE than the estimate that minimizes the MSE without the perceptual constraint. 
However, in our setting we have the additional causality constraint \eqref{eq:causality}.
Requiring both constraints \eqref{eq:perfect_perception} and \eqref{eq:causality} might incur an additional cost, as illustrated by the following example, where applying each one of them does not restrict the optimal solution, but together they result in a higher MSE.
\begin{example}
Let $T=1$ and consider the process $(x_0,x_1)=(q_0,q_0)$, where $q_0 \sim \mathcal{N}(0,1)$, with observations $(y_0,y_1) = (0,x_1)$.
Assume we want to minimize the error at time $k=1$ (namely $(\alpha_0,\alpha_1)=(0,1)$ in \eqref{eq::cost_C:def}). Then, considering only the causality constraint \eqref{eq:causality}, the estimator $(\hat{x}_0,\hat{x}_1)=(y_0,y_1)$ is optimal. Indeed, it is causal and it achieves zero MSE. Similarly, considering only the perceptual quality constraint \eqref{eq:perfect_perception}, the estimator $(\hat{x}_0,\hat{x}_1)=(y_1,y_1)$ is optimal. Indeed, it is distributed like $(x_0,x_1)$ and it also achieves zero MSE. 
However, when demanding both conditions, $\hat{x}_0$ must be based on no information (as $y_0=0$), and it must be drawn from the prior distribution $\mathcal{N}(0,1)$ in order to be distributed like $x_0$. Furthermore, to obey \eqref{eq:perfect_perception}, we must have $\hat{x}_1=\hat{x}_0$. Therefore, the optimal estimator in this case is $(\hat{x}_0,\hat{x}_1)=(\tilde{q}_0,\tilde{q}_0)$, where $\tilde{q}_0 \sim \mathcal{N}(0,1)$ is independent of $q_0$. The MSE achieved by this estimator is $2$.
\end{example}

\section{Perfect perceptual-quality filters}

The perceptual constraint \eqref{eq:perfect_perception} dictates that the estimator must be of the form
\begin{equation}
\hat{x}_{k}=A_{k}\hat{x}_{k-1}+J_{k},\quad \hat{x}_0 = J_0,\label{eq:xperc::form}
\end{equation}
where $J_{k}=\hat{x}_{k}-A_{k}\hat{x}_{k-1}$ is distributed as $\mathcal{N}(0,Q_{k})$ and is independent
of $\hat{X}_{0}^{k-1}$. 
Note the similarity between \eqref{eq:xperc::form} and the MSE-optimal state \eqref{eq::kalman:recurrence}, in which $J_k=K_k\II_k$. Here, however, this choice is not valid due to the constraint on the output distribution. 
In terms of temporal consistency, an estimator of the form \eqref{eq:xperc::form} 
guarantees that previously presented features obey the natural dynamics of the domain, while newly generated estimates do not contradict the previous ones.
In order
 to maintain causality \eqref{eq:causality}, $J_k$ must be of the form
\begin{equation}
J_{k}\sim p_{J_{k}}(\cdot|y_{0},\ldots,y_{k},\hat{x}_{0},\ldots,\hat{x}_{k-1}),
\label{eq::pkl:Jk is causal}
\end{equation}
\textit{i.e.}, $J_k$ is independent of future measurements $Y_{k+1}^T$ given $Y_0^k$. As a consequence, $J_k$ is uncorrelated with $\II_{k+n}$ for all $n\geq1$. 

We now discuss linear estimators, where $Y_0^T$ and $J_0^T$ (hence $\hat{X}_0^T$) are jointly Gaussian. Our first result is as follows (see proof in App.~\ref{app:optimality_lin}). 
\begin{theorem}\label{thm:generalFilter}
Under the cost \eqref{eq::cost_C:def}, there exists an optimal linear estimator of the form
\begin{equation}
    J_k = \pi_k \II_k + \phi_k \UII_k +w_k, \label{eq::pkl::Jk:innovatonUpsilon}
\end{equation}
$\pi_k\in \R^{n_x \times n_y}$ and $\phi_k\in \R^{n_x \times (kn_y)}$ are the filter's coefficients, $w_k$ is an independent Gaussian noise with covariance
    $\Sigma_{w_k} = Q_k - \pi_k S_k \pi_k^\T - \phi_k \Sigma_{\UII_k} \phi_k^\T \succeq 0$, $v_k$ is the process of \emph{unutilized information}
\begin{equation}
    \UII_k \triangleq \II_0^{k-1} - \EE\left[\II_0^{k-1}\middle|\hat{X}_{0}^{k-1} \right], \, \UII_0=0.
    \label{eq::pkl:uii_k}
\end{equation}
\end{theorem}
Note that with this form for $J_k$, the state $\hat{x}_k$ is indeed a function of the observations $Y_0^k$ and the previous states $\hat{X}_0^{k-1}$. Intuitively, $\UII_k$ is the part of the information in the observations, which has no correlation with the information used to construct the past estimates  $\hat{X}_0^{k-1}$. Thus, from the standpoint of the filter's output, this information has not yet been introduced. As opposed to the innovation $\II_k$, the process $\UII_k$ is not white, and it is affected by the choices of $\pi_t$ and $\phi_t$ up to time $k-1$. However, $\II_k$ is always independent of $\UII_k$, since $\II_k$ is independent of $\II_0^{k-1}$ and $J_0^{k-1}$, which constitute $\UII_k$.

 \begin{wrapfigure}{r}{0.62\linewidth}
\centering \includegraphics[viewport=65bp 120bp 890bp 470bp,clip,width=1.0\linewidth]{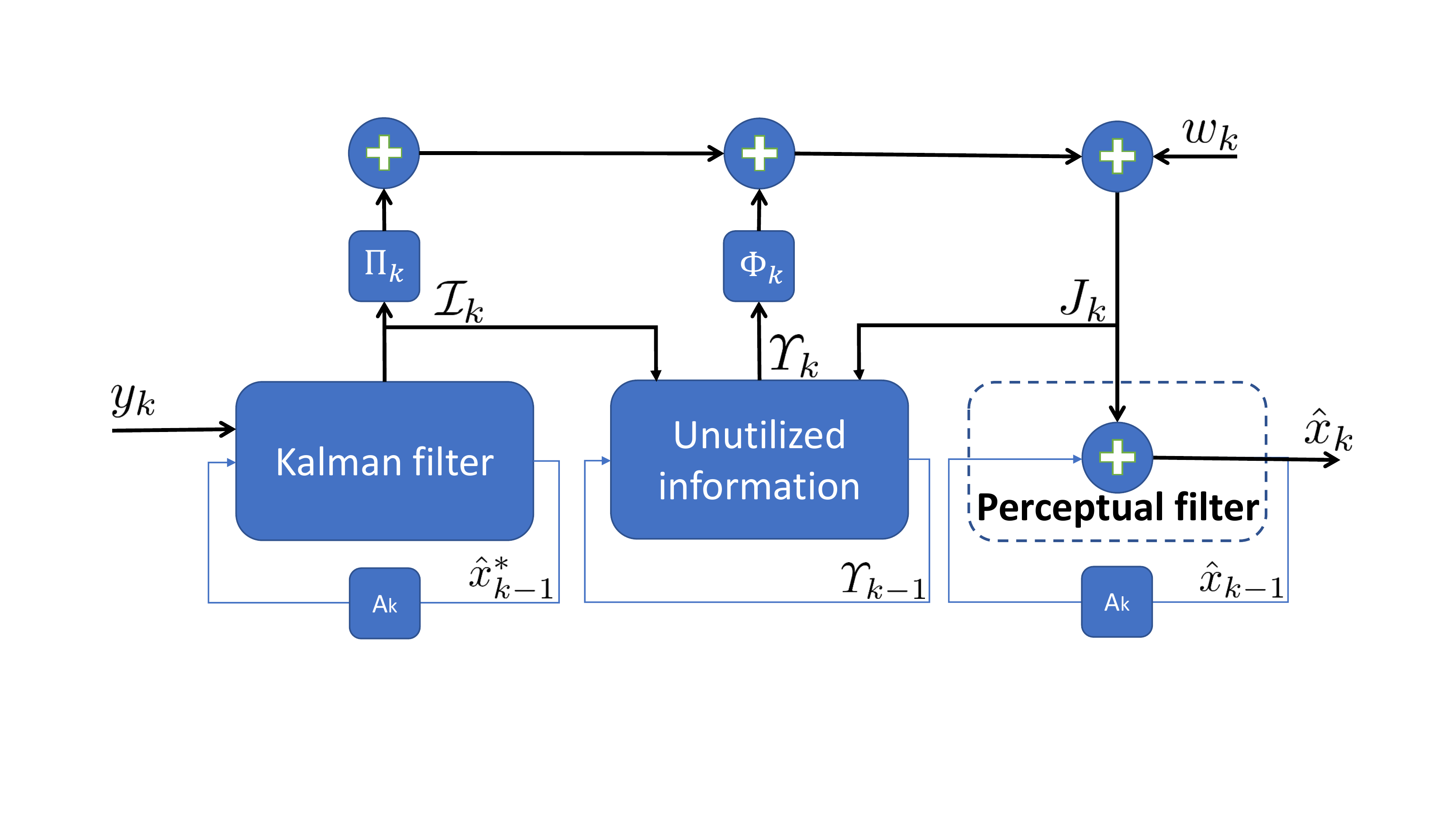}
\caption[Linear perceptual filtering.]{\textbf{Recursive perceptual filtering (Sec.~\ref{subsec:Error-analysis}).} \label{fig:filter sketch} The state estimator $\hat{x}_k$ consists of the previous state $\hat{x}_{k-1}$, and the innovation and unutilized information processes. The unutilized information state $\varUpsilon_{k+1}$ is updated using the previously unutilized information $\varUpsilon_k$ 
and the newly-arriving information $\II_k$. The currently utilized information, arriving from $J_k$, is then subtracted from $\varUpsilon_{k+1}$. 
}
\end{wrapfigure}

The filter of Theorem~\ref{thm:generalFilter} is causal but not recursive. Specifically, although it is possible to obtain $\UII_{k+1} , \Sigma_{\UII_{k+1}}$ given the coefficients $\{\pi_t,\phi_t\}_{t=0}^k$ (see App.~\ref{app:filter_deriv}), the dimension of $v_k$ grows with time (it is $k n_y$), thus increasing the cost of computing $\phi_k v_k$. Furthermore, determining the coefficients $\{\pi_k,\phi_k\}_{k=0}^T$ that minimize the objective \eqref{eq::cost_C:def} (which is done a-priori in an offline manner) requires solving a large optimization problem, as the total size of all coefficients is $O(n_xn_yT^2)$. To perform an efficient optimization over these coefficients, we next suggest two simplified versions of this form, which may generally be sub-optimal but easier to optimize.

\subsection{Recursive-form filters}
\label{subsec:Error-analysis}

Recall that the optimal Kalman state $\hat{x}_{k}^{*}$ achieves the minimal possible MSE, which is given by 
$d_{k}^{*}=\mathbb{E}\left[\|\hat{x}_{k}^{*}-x_{k}\|^{2}\right]=\tr{ P_{k|k} }$,
where $P_{k|k}$ is the error covariance, given explicitly in Alg.~\ref{alg:Kalman-Filter}. By the orthogonality principle, any other estimator $\hat{x}_{k}$ based
on the observations $Y_0^k$, satisfies
\begin{equation}\label{eq:dStarPlusMSE}
\EEb{\|x_{k}-\hat{x}_{k}\|^{2}}
=d_{k}^{*}+\EEb{\|\hat{x}_{k}-\hat{x}_{k}^{*}\|^{2}}.
\end{equation}
Now, consider an estimator $\hat{x}_{k}$ of the form (\ref{eq:xperc::form}),
and let 
$D_{k}\triangleq \EE[ (\hat{x}_{k}^{*}-\hat{x}_{k})(\hat{x}_{k}^{*}-\hat{x}_{k})^{\T}]$.
Since we choose $J_{k}$ to be normally distributed and independent
of $\hat{x}_{k-1}$, it is easy to see that $D_{k}$ obeys the \textit{Lyapunov
difference equation}
\begin{flalign}
D_{k}  \!= \! A_{k}D_{k-1}A_{k}^{\T}\!+\!K_{k}S_{k}K_{k}^{\T}\!+\!Q_{k} \!-\!\EE[J_{k}\II_{k}^{\T}]K_{k}^{\T}\!-\!K_{k}\EE[\II_{k}J_{k}^{\T}]
 \!-\!A_{k}\EE[\hat{x}_{k-1}^{*}J_{k}^{\T}]\!-\!\EE[J_{k}\hat{x}_{k-1}^{*\T}]A_{k}^{\T}.
 \label{eq:D_k::Lyapunov_general}
\end{flalign}
As we see, the choice of $J_{k}$ affects current (and future)
errors by its correlation with the two independent
components,
$
(A_{k}\hat{x}_{k-1}^{*},K_{k}\II_{k})
$. 
Let us now consider filters of the form
\begin{equation} \label{eq::pkl:J_k::fullJk}
J_{k}=\varPhi_{k}A_{k}\Upsilon_{k}+\Pi_{k}K_{k}\II_{k}+w_{k},\quad w_{k}\sim\mathcal{N}\left(0,\Sigma_{w_{k}}\right) ,
\end{equation}
where here, by slight abuse of notation, we define the process of unutilized information as
\begin{equation} \label{eq::pkl:ups_k}
\Upsilon_{k}\triangleq\hat{x}_{k-1}^{*}-\mathbb{{E}}\left[\hat{x}_{k-1}^{*}|\hat{x}_{0},\ldots,\hat{x}_{k-1}\right],
\quad \Upsilon_0=0.
\end{equation}
The matrices $\Pi_{k},\varPhi_{k}\in \R^{n_x \times n_x}$ are coefficients such that 
\begin{equation} \label{eq::pkl:Sigma_w_full}
\Sigma_{w_{k}}=Q_{k}-\varPhi_{k}A_{k}\Sigma_{\Upsilon_{k}}A_{k}^{\top}\varPhi_{k}^{\top}-\Pi_{k}M_{k}\Pi_{k}^{\top}\succeq0,
\end{equation}
where we denote the Kalman update covariance by
$M_{k}\triangleq K_{k}S_{k}K_{k}^{\T}$. This guarantees that $J_k\mathcal\sim {N}(0,Q_{k})$, as desired.
Importantly, as opposed to $v_k$, the dimension of $\Upsilon_{k}$ is \emph{fixed}, namely it does not grow with time $k$. 
Note that since $\hat{x}_{k-1}^*$ is a linear combination of $(\II_0,\ldots,\II_{k-1})$, \eqref{eq::pkl:J_k::fullJk} is a special choice of $\pi_k$ and $\phi_k$ in \eqref{eq::pkl::Jk:innovatonUpsilon} where coefficient size does not grow with $k$ as well. $\varUpsilon_k$ and its covariance $\Sigma_{\varUpsilon_{k}}$ are given via a recursive form, illustrated in Fig.~\ref{fig:filter sketch} (and derived in App.~\ref{app:filter_deriv}):
\begin{equation}
\label{eq:pkf::recursiveUpsSigma}
    \varUpsilon_{k+1} = A_k\varUpsilon_{k} +K_k \II_k -  \varPsi_k Q_k^\dagger J_k, 
\,
\Sigma_{\varUpsilon_{k+1}} 
= A_k \Sigma_{\varUpsilon_k} A_k^\top +M_k -\varPsi_k      Q_k^\dagger \varPsi_k^\T,
\end{equation}
\begin{equation}
\label{eq:pkf::recursivePsi}
\varPsi_k = M_k \Pi_k^\top +A_k \Sigma_{\varUpsilon_k} A_k^\T \varPhi_k^\top. 
\end{equation} 

Note again that unlike the innovation $\II_k$, $\Upsilon_{k}$ might not be a white process, but we have
that $\Upsilon_{k}$ is independent of the filter's output $\hat{X}_{0}^{k-1}$ and $\II_k$. Equation \eqref{eq:D_k::Lyapunov_general} now takes the form
\begin{flalign} \label{eq:D_k::Lyapunov_fullJk}
D_{k} =&A_{k}D_{k-1}A_{k}^{\top}+Q_{k}+M_{k} -\Pi_{k}M_{k}-M_{k}\Pi_{k}^{\top}
 -A_{k}\Sigma_{\Upsilon_{k}}A_{k}^{\top}\varPhi_{k}^{\top}-\varPhi_{k}A_{k}\Sigma_{\Upsilon_{k}}A_{k}^{\top},
\end{flalign}
where we observe that $\Sigma_{\Upsilon_{k}}$ may depend on the choice of $\left\{ \Pi_{t},\varPhi_{t}\right\} _{t=0}^{k-1}$.  In order to retrieve an optimal filter, one should
perform optimization of the desired objective over $\left\{ \Pi_{t},\varPhi_{t}\right\} _{t=0}^{T}$,
under the constraints given in \eqref{eq::pkl:Sigma_w_full}. From \eqref{eq:dStarPlusMSE}, minimizing the cost \eqref{eq::cost_C:def} boils down to minimizing $\sum_{k=0}^T\alpha_k \tr{D_k}$ subject to the constraints in \eqref{eq::pkl:Sigma_w_full}, which is an optimization problem over only $O(n_x^2T)$ parameters. 
\subsection{An Exactly solvable reduction: Perceptual Kalman Filter}
\label{sec::Perceptual Kalman filters}

We now consider an additional reduction, which allows to obtain a closed form solution for the filter's coefficients. Specifically, a reduced-size filter can be obtained by using the form \eqref{eq:xperc::form} and \eqref{eq::pkl:J_k::fullJk}
with the sub-optimal choice $\varPhi_k\equiv0$, namely
\begin{equation}
J_{k}=\Pi_{k}K_{k}\II_{k}+w_{k}.\label{eq:PKF:Jkform}
\end{equation}
The meaning of this choice is that only newly-observed information is used for updating estimation at each stage, while non-utilized information from previous time steps is discarded. Here, $\Pi_{k}$ is a $n_{x}\times n_{x}$ coefficient matrix, and
$w_{k}\sim\mathcal{N}(0,Q_{k}-\Pi_{k}M_{k}\Pi_{k}^{\T})$
is a Gaussian noise, uncorrelated with all other states, observations
and noises in the system up to time $k$. Again, note that $\II_{k}$
is independent of the measurements up to time $k-1$, hence this choice
makes $J_{k}$ independent of $\hat{X}_{0}^{k-1}$.
These innovation-based corrections resemble the mechanism exploited in \eqref{eq::kalman:recurrence},
hence we will refer to optimal filters of the form \eqref{eq:PKF:Jkform} as $\emph{perceptual}$ Kalman filters (PKF).

Now, by a straightforward substitution, \eqref{eq:D_k::Lyapunov_general}
becomes
\begin{equation}
D_{k}=A_{k}D_{k-1}A_{k}^{\T}+Q_{k}+M_{k}-\Pi_{k}M_{k}-M_{k}\Pi_{k}^{\T},\quad k=0,\ldots, T,\label{eq:D_k::Lyapunov_PKF}
\end{equation}
where we consider $Q_{0}=P_{0}$, $M_{0}=\Sigma_{\hat{x}_{0}^{*}}$ and
$D_{-1}=0$. As before, minimizing \eqref{eq::cost_C:def} boils down to minimizing $\sum_{k=0}^T\alpha_k \tr{D_k}$, 
and in order for \eqref{eq:PKF:Jkform} to be well-defined,
we should enforce the constraints 
$Q_{k}-\Pi_{k}M_{k}\Pi_{k}^{\T}\succeq0,\,k=0,\ldots,T$.
For simplicity, we consider the time-invariant case where $A_{k}\equiv A$, so that the optimization objective becomes
\begin{equation}
\begin{cases}
\min_{\{\Pi_{k}\}_{k=0}^{T}}&\sum_{k=0}^{T}\w_{k}\tr{D_{k}}  \\
\st & 
D_k = \sum_{t=0}^k A^{k-t} \left[ \mathscr{Q}_t -\Pi_{t}M_{t}-M_{t}\Pi_{t}^{\T} \right] \left(A^\T\right)^{k-t}, \quad k=0,\ldots,T,
\\ & 
Q_{k}-\Pi_{k}M_{k}\Pi_{k}^{\T}\succeq0,\quad k=0,\ldots,T, 
\end{cases}\label{eq:Agen:traceproblem}
\end{equation}
where we denoted $\mathscr{Q}_k=Q_k+M_k$.
Substituting $D_k$, we can rewrite the objective as
\begin{align}
\sum_{k=0}^T \tr{  \sum_{t=k}^T \w_t A^{t-k}\left[ \mathscr{Q}_k -2\Pi_{k}M_{k} \right] \left(A^\T\right)^{t-k}}. 
\end{align}
As we can see, optimization over a particular coefficient $\Pi_k$ does not affect other summands of the external sum. Therefore, each $\Pi_k$ can be optimized separately. Minimizing the cost at the $k$-th step is equivalent to
\begin{equation}
\max_{\Pi_k}\tr{\Pi_{k}M_{k}\sum_{t=k}^{T}\w_{t}(A^{t-k})^{\T}A^{t-k}} \quad \mathrm{s.t.}\quad Q_{k}-\Pi_{k}M_{k}\Pi_{k}^{\T}\succeq0.
\label{eq:stepk:traceproblem}
\end{equation}
Let us denote 
$B_{k}\triangleq\sum_{t=k}^{T}\w_{t}(A^{t-k})^{\T}A^{t-k}=\w_kI+A^\T B_{k+1} A$. As we now show, this optimization problem possesses a closed-form solution under a mild assumption (which is satisfied \text{e.g.}~when $Q_{k}\succ0$). The proof is given in Appendix~\ref{app::extremalSDP}.

\begin{theorem}
\label{Thm:pkf:optimal_perc_gain} Assume that $\im{B_k M_{k} B_k}\subseteq \im{Q_k}$ for every $k$. Let $M_{B}\triangleq B_kM_kB_k$ and denote $\Omega=\left\{\Pi_k: Q_{k}-\Pi_{k}M_{k}\Pi_{k}^{\T}\succeq0\right\}$. 
Then the optimal value in \eqref{eq:stepk:traceproblem} is given by
\begin{equation}
\max_{\Pi_k\in\Omega}\tr{ \Pi_k M_k B_k} =\tr{ \left(M_{B}^{1/2}Q_kM_{B}^{1/2}\right)^{\hlf}} ,
\end{equation}
and is achieved by the optimal coefficient (which is generally not unique)
\begin{equation}
\Pi_k^{*}=Q_kM_{B}^{1/2}\left(M_{B}^{1/2}Q_kM_{B}^{1/2}\right)^{1/2\dagger}M_{B}^{\dagger1/2}B_k.\label{eq:stepk:optimalPercepGain}
\end{equation}
\end{theorem}
For a closed form solution under the alternative assumption that
$\im{M_{k}}\subseteq \im{Q_k}$, as well as a discussion of stationary filters, please see the Appendix. The Perceptual Kalman filter (PKF) obtained from Thm.~\ref{Thm:pkf:optimal_perc_gain} is summarized in Alg.~\ref{alg:Perceptual-Kalman-Filter}. 

\begin{algorithm}[t]
\caption{Perceptual Kalman Filter (PKF) \label{alg:Perceptual-Kalman-Filter}}
\begin{algorithmic}

\STATE {\bfseries Input}: Kalman Filter outputs $\{\II_{k},K_{k},S_{k}\}_{k=0}^{T}$,
weights $\{\w_{k}\}_{k=0}^{T}$,  matrices $A,P_0,\{Q_k\}_{k=1}^{T}$ satisfying $\im{B_k M_{k} B_k}\subseteq \im{Q_k}$ for all $k=0\ldots,T$. 

{\bfseries initialize}: $\hat{x}_{0}=\Pi_{0}K_0y_{0}+w_{0}$, $w_{0}\sim \mathcal{N}\left(0,P_{0}-\Pi_{0}M_{0}\Pi_{0}^{\T}\right)$, where $\Pi_{0}$ is given by \eqref{eq:stepk:optimalPercepGain} and $B_{0}=\sum_{t=0}^{T}\w_{t}(A^{t})^{\T}A^{t}$. 

\FOR{$k=1$ {\bfseries to} $T$ }

\STATE \textbf{calculate}: $M_{k}=K_{k}S_{k}K_{k}^{\T}$, $B_{k}=\sum_{t=k}^{T}\w_{t}(A^{t-k})^{\T}A^{t-k}$,
$M_{B}=B_{k}M_{k}B_{k}$.

\STATE \textbf{compute optimal gain} \eqref{eq:stepk:optimalPercepGain}: 
$\Pi_{k}=Q_{k}M_{B}^{\hlf}\left(M_{B}^{\hlf}Q_{k}M_{B}^{\hlf}\right)^{\hlf\dagger}M_{B}^{\dagger\hlf}B_{k}M_{k}M_{k}^{\dagger}
$.

\STATE \textbf{sample}: $w_{k}\sim \mathcal{N}\left(0,Q_{k}-\Pi_{k}M_{k}\Pi_{k}^{\T}\right).$

\STATE \textbf{update state}: $\hat{x}_{k}=A_{}\hat{x}_{k-1}+\Pi_{k}K_{k}\II_{k}+w_{k}$.
\ENDFOR
\end{algorithmic}
\end{algorithm}

%% file: pk_demo_short.tex
\newpage
\section{Numerical demonstrations}
\label{sec::numericals}

\definecolor{Gray}{gray}{0.9}
\newcolumntype{g}{>{\columncolor{Gray}}c}

\begin{wraptable}{R}{0.62\textwidth}
\caption{List of demonstrated filters.}
    \label{tab::pkl:filters demonstrated}
    \centering
    \small
    \begin{tabular}{ccccc}
         & \multirow{2}{*}{\textbf{description}} & \multirow{2}{*}{\textbf{def.}}& \multicolumn{2}{c}{\textbf{perfect-perception}} 
         \\ 
         & & &per-sample&temporal
         \\
         \hline
         \rowcolor{Gray} $\hat{x}_{\kal}^*$& Kalman filter & Alg.\ref{alg:Kalman-Filter} & \xmark & \xmark 
         \\
         \multirow{2}{*}{$\hat{x}_{\ntc}$} & Per-sample quality &\multirow{2}{*}{\eqref{eq::pkl:ntc-estimator} } & \multirow{2}{*}{ \cmark }& \multirow{2}{*}{\xmark} 
         \\
         &(no temporal)& & &
         \\
         \rowcolor{Gray} {$\hat{x}_{\opt}$} & Optimized filter & {\eqref{eq::pkl:J_k::fullJk} }& {\cmark} & {\cmark} 
         \\
         \multirow{2}{*}{$\hat{x}_{\auc}$} & PKF & \multirow{2}{*}{Alg.\ref{alg:Perceptual-Kalman-Filter}}& \multirow{2}{*}{\cmark} & \multirow{2}{*}{\cmark} 
         \\
         &(total cost)& & &
         
         \\
         \rowcolor{Gray}  & PKF & &  & 
         \\\rowcolor{Gray}
         \multirow{-2}{*}{$\hat{x}_{\minT}$}&(terminal cost)& \multirow{-2}{*}{Alg.\ref{alg:Perceptual-Kalman-Filter}} & \multirow{-2}{*}{\cmark} & \multirow{-2}{*}{\cmark} 
         
         \end{tabular}
\end{wraptable}

We now revisit our main question: \textit{what is the cost of temporal consistency in online restoration?} In addition, as we have seen in Sec. \ref{sec::Perceptual Kalman filters}, the relaxation $\varPhi_k=0$, yielding the Perceptual Kalman filters, reduces the complexity of computation, possibly at the cost of higher errors. It is natural, then, to ask what is the cost of this simplification.
In the following experiments, we compare the performance of several filters; $\hat{x}^*_{\kal}$ and $\hat{x}_{\ntc}$  correspond to the Kalman filter and the temporally-inconsistent filter \eqref{eq::pkl:ntc-estimator} (which does not possess perfect-perceptual quality). The estimate 
 $\hat{x}_{\opt}$ is generated by a  perfect-perception filter obtained by numerically optimizing the coefficients in \eqref{eq::pkl:J_k::fullJk}, where the cost is the MSE at termination time, $\C_\mathrm{T}=\EE[\|\hat{x}_{T}-x_{T}\|^{2}]$. The estimates
$\hat{x}_{\auc},\hat{x}_{\minT}$ correspond to PKF outputs (Alg.~\ref{alg:Perceptual-Kalman-Filter}) minimizing the \textit{total cost} (area under curve) $\C_{\auc}=\sum_{k=0}^{T}\EE[\|\hat{x}_{k}-x_{k}\|^{2}]$ and the \textit{terminal cost} $C_T$, respectively.
The filters are summarized in Table \ref{tab::pkl:filters demonstrated}. Full details and additional experimental results are given in App.~\ref{app:numerical}.

\subsection{Harmonic oscillator}

\begin{wrapfigure}{R}{0.6\linewidth}
\centering \centering 
\includegraphics[viewport=12bp 10bp 1380bp 675bp,clip,width=1.0\linewidth]
{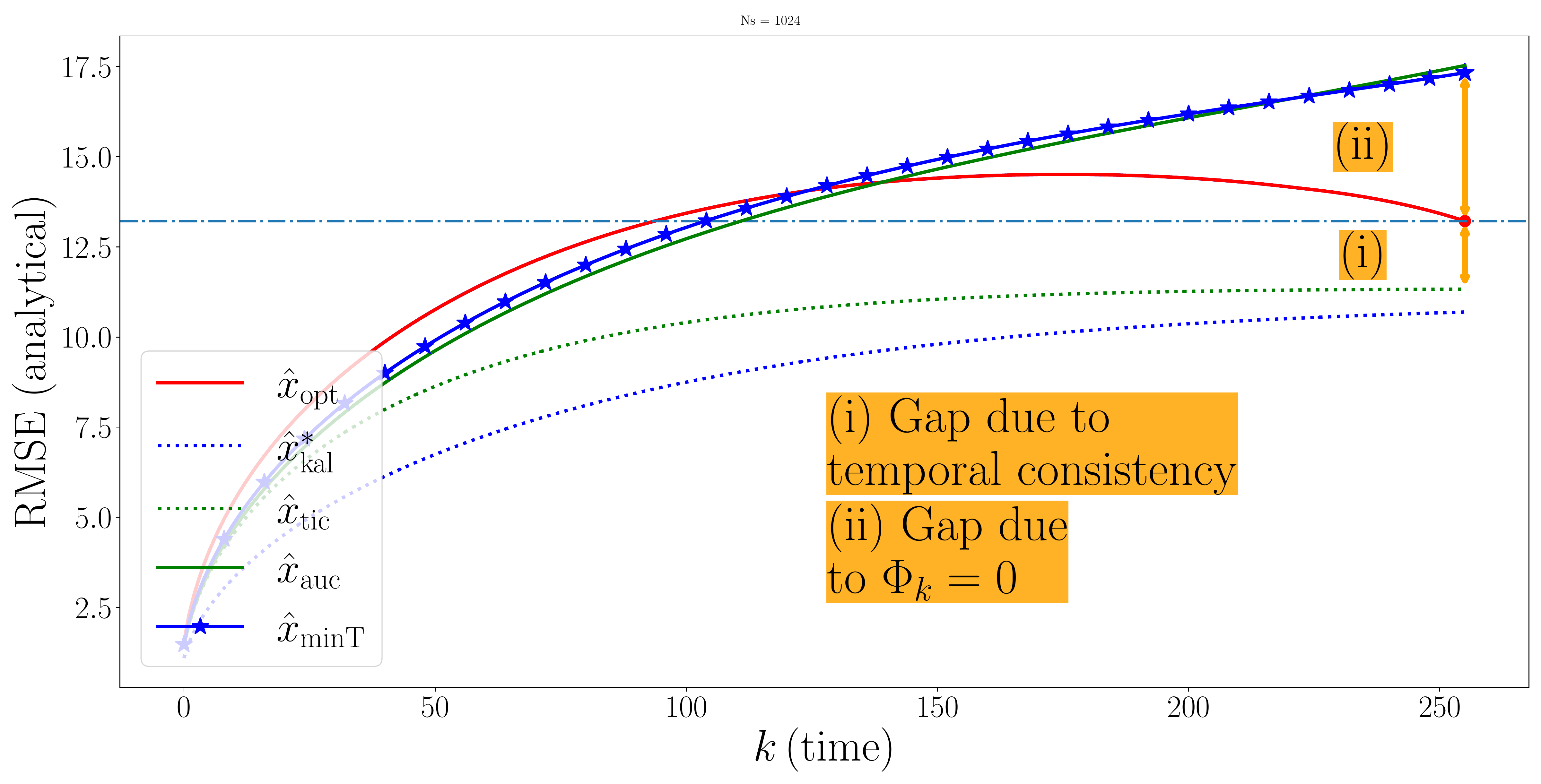}
\caption{
\textbf{MSE on Harmonic oscillator.} \label{fig::harmonic_osc:MSE_fullview}
We observe the difference in distortion between the perfect-perceptual state $\hat{x}_{\opt}$, optimized according to \eqref{eq::pkl:J_k::fullJk}, and $\hat{x}_{\ntc}$. This additional cost is due to the perceptual constraint on the joint distribution. Also note the gap between MSE of the optimized estimator and $\hat{x}_{\minT}$ which is due to the sub-optimal choice of coefficients, $\varPhi_k=0$.
}
\end{wrapfigure}

We start with a simple $2$-D example.
Specifically, we consider a harmonic oscillator, where the state $x_k \in \R^2$ corresponds to position and velocity, and the observation at time $t$ is a noisy versions of the position at time $t-\tfrac{1}{2} \Delta_t$, where $\Delta_t$ is the sampling period (see App.~\ref{app:numerical} for details). Figure~\ref{fig::harmonic_osc:MSE_fullview} shows the MSE for $\hat{x}_{\opt}$ and the sub-optimal PKF outputs $\hat{x}_{\auc},\hat{x}_{\minT}$. We observe that the PKF estimations are indeed not MSE optimal at time $T$. However, their RMSE at time $T$ is only $\sim 30\%$ higher than that of $\hat{x}_{\opt}$ and they have the advantage that they can be solved analytically and require computing only half of the coefficients ($\Pi_k$).
The estimates $\hat{x}^*_{\kal}$ and $\hat{x}_{\ntc}$ achieve lower MSE than $\hat{x}_{\opt}$, however they do not possess perfect-perceptual quality. \textit{The difference in MSE between the
filters $\hat{x}_{\opt}$ and $\hat{x}_{\ntc}$ is the cost of temporal consistency in online estimation for this setting}.

\input{dyntex}

%% file: dyntex.tex
\subsection{Dynamic texture}

We now illustrate the qualitative effects of perceptual estimation in a simplified video restoration setting.
Specifically, we consider a video of a ``dynamic texture'' of waves in a lake. Such dynamic textures are accurately modeled by linear dynamics of a Gaussian latent representation \citep{doretto2003dynamic}, whose parameters we learn from a real video. Here, frames are generated from a latent $128$-dimensional state $x^{FA}_k$ which corresponds to their \textit{Factor-Analysis} (FA) decomposition (see \textit{e.g.} \citep[Sec. 12.2.4]{bishop2006pattern} for more details).  
Thus, $512 \times  512\times3$ frames in the video domain are created through an affine transformation of $x^{FA}_k$. 
Linear observations $y_k \in \R^{32\times32}$ are given in the frame (pixel) domain, by 
$16\times$ downsampling the $Y$-channel of the generated ground-truth frames, and adding white Gaussian noise.  All filtering is done in the latent domain, and then transformed to the pixel domain. MSE is also calculated in the FA domain. The exact settings can be found in App.~\ref{app:numerical}.

In the first  experiment, measurements are supplied up to frame $k=127$ and then stop (Fig. \ref{fig::pkf:river_demo:predictexp}), letting the different filters predict the next, unobserved, frames of the sequence. 
We can see that until frame $k = 127$, all filters reconstruct the reference frames well. Starting from time $k=128$, when measurements stop, the Kalman filter slowly fades into a static, blurry output which is the average frame value in this setting. This is a non-`realistic' video; Neither the individual frames nor the temporal (static) behavior are natural to the domain. 
Our perfect-perceptual filter, $\hat{x}_{\auc}$, keeps generating a `natural' video, both spatially and temporally. This makes its MSE grow faster\footnote{More visual details, including ground-truth clips and empirical error can be found in the Appendix. Please see the supplementary video for the full videos.}.

\begin{figure}
\centering \centering 
\includegraphics[viewport=12bp 300bp 1155bp 420bp,clip,width=0.85\linewidth]{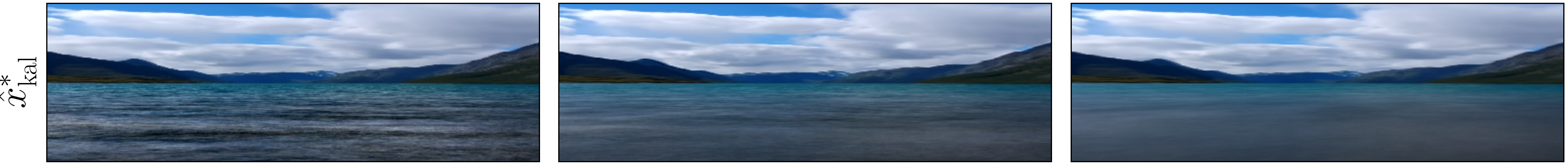}
\includegraphics[viewport=12bp 10bp 1155bp 170bp,clip,width=0.85\linewidth]{plots/river/river_wc10_ours.pdf}
\includegraphics[viewport=12bp 10bp 1155bp 350bp,clip,width=0.77\linewidth]{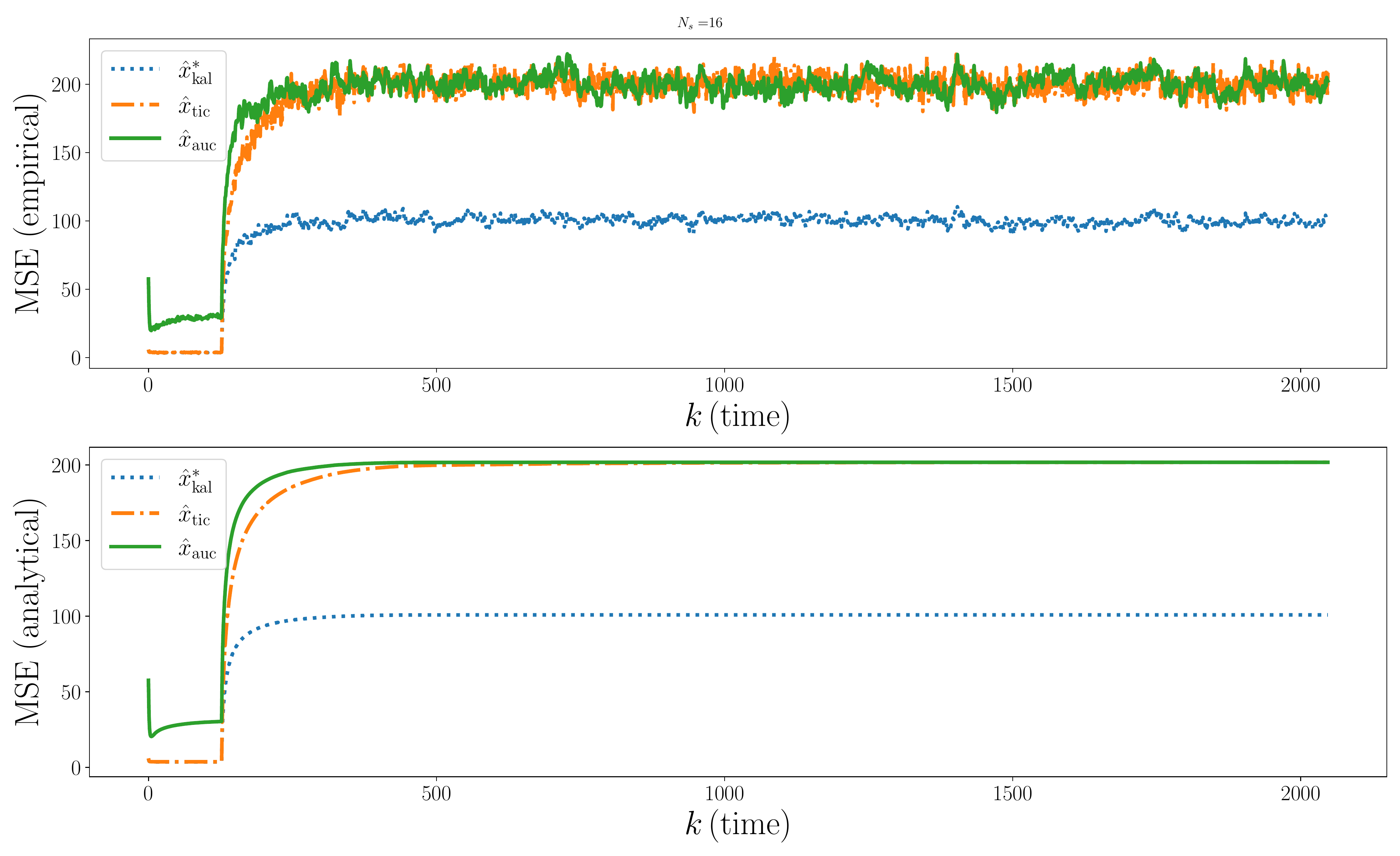}
\caption{\label{fig::pkf:river_demo:predictexp}
\textbf{Frame prediction on a dynamic texture domain.}   
In this experiment, measurements are supplied only up to frame $k=127$ and the filter's task is to predict the unobserved future frames.  Observe that $\hat{x}^*_{\kal}$ fades into a  blurred average frame, while the perceptual filter $\hat{x}_{\auc}$ generates a natural video, both spatially and temporally. This makes its MSE grow faster.
}     
\end{figure}

We now perform a second experiment, where measurements are set to zero until frame $k=512$.  At times $k > 512$ they are given again by the noisy, downsampled frames as described above.
In Fig. \ref{fig::pkf:river_demo:genexp} we present the outcomes of the different filters. 
We first note that up to frame $k=512$, there is no observed information, hence outputs are actually being generated according to priors.
The Kalman filter outputs  a static, average frame. The filter
$\hat{x}_{\ntc}$ randomizes each frame independently, leading to unnatural random movement with flickering features. 
At frame $k=513$, when observations become available, $\hat{x}^*_{\kal}$ and $\hat{x}_{\ntc}$ get updated immediately, creating an inconsistent, non-smooth motion between frames $512$ and $513$. The PKF output $\hat{x}_{\auc}$, on the other hand, maintains a smooth motion.
Since the outputs of inconsistent filters rapidly becomes similar to the ground-truth, their errors drop. The perfect-perceptual filter, $\hat{x}_{\auc}$, remains consistent with its previously generated frames and the natural dynamics of the model, hence its error decays more slowly.

\begin{figure}
\centering \centering 
\includegraphics[viewport=12bp 10bp 1155bp 420bp,clip,width=0.85\linewidth]{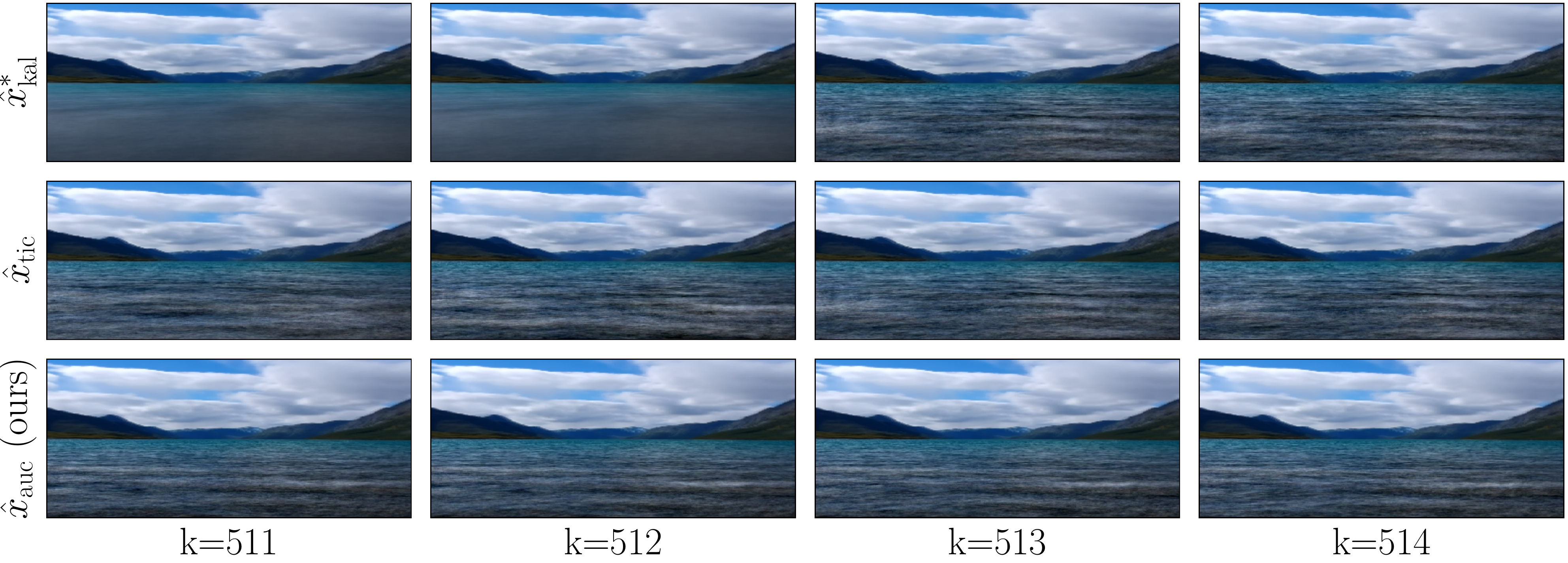}
\includegraphics[viewport=12bp 10bp 1155bp 350bp,clip,width=0.77\linewidth]{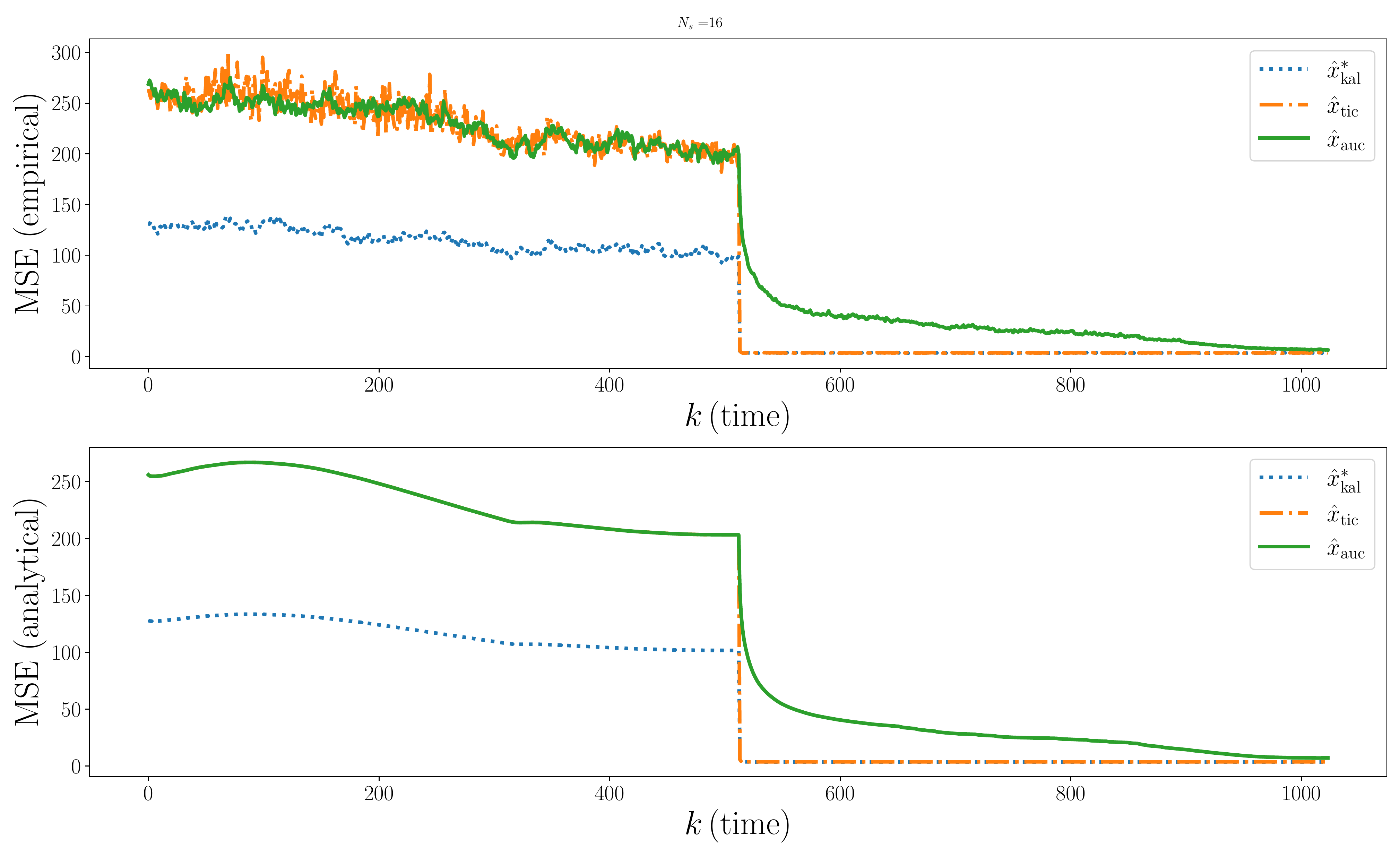}
\caption{\label{fig::pkf:river_demo:genexp}
\textbf{Frame generation on a dynamic texture domain.}   
In the first half of the demo ($k\leq 512$), there are no observations, hence the reference signal is restored according to prior distribution. The filters with no perfect-perceptual quality constraint in the temporal domain generate non-realistic frames (Kalman filter output $\hat{x}_{\kal}^*$) or unnatural motion ($\hat{x}_{\ntc}$). Perceptual filter $\hat{x}_\auc$ is constrained by previously generated frames and the natural dynamics of the domain, hence its MSE decays slower.
}     
\end{figure}
\newpage

%% file: pkl_appendix.tex
\begin{center}
\hrule height 4pt
\vskip 0.25in 
\vskip -\parskip%
{  \bf \LARGE \centering Perceptual Kalman Filters: Online State Estimation under a Perfect Perceptual-Quality Constraint - Supplementary Material}
 \vskip 0.29in 
  \vskip -\parskip
  \hrule height 1pt
  \vskip 0.09in%
  \end{center}

In App.~\ref{app::kalman filter} we provide a detailed theoretical background on the Kalman Filter, its properties and recursive calculation.
In App.~\ref{app:optimality_lin} we prove that under our perfect perceptual filtering setting, there exists a \emph{linear} optimal filter (Thm.~\ref{thm:generalFilter}).
In App.~\ref{app::direct approach} we discuss a direct, non-recursive method for optimizing perceptual filter coefficients.
In App.~\ref{app:filter_deriv} we derive the recursive expression for the filter given in \eqref{eq::pkl:J_k::fullJk}.
In App.~\ref{app::extremalSDP} we find a closed-form solution for PKF coefficients by proving Theorem~\ref{Thm:pkf:optimal_perc_gain}. In this appendix, we also give some brief overview on the extremal problem of finding a minimal distance between distributions.
App.\ref{app::stationary} contains a discussion about stationary perceptual Kalman filters in
the steady-state regime.
We summarize all definitions and notations in App.~\ref{app::notation}.
Finally, in App.\ref{app:numerical} we give full details for all numerical demonstrations, and present additional empirical and visual results. More results are provided in the supplementary video.

\section{The Kalman Filter algorithm}
\label{app::kalman filter}
In this Section we supply a detailed reminder of the Kalman filter Algorithm.
The celebrated Kalman filter \citep{kalman1960filtering} assumes a state
$x_{k}\in\mathbb{R}^{n_{x}}$, where dynamics are modeled as deterministic
linear functions perturbed by a Gaussian noise, and observations $y_k \in \R^{n_y}$ are
linear functions of $x_{k}$ with an additive noise
\begin{align}
x_{k}=&A_{k}x_{k-1}+q_{k},&& q_{k}\sim\mathcal{N}(0,Q_{k}),&& k=1,...,T, \label{appeq:kalman:setting1} \\
y_{k}=&C_{k}x_{k}+r_{k},&& r_{k}\sim\mathcal{N}(0,R_{k}),&& k=0,...,T.\label{appeq:eq:kalman:setting2}
\end{align}
The noise terms $q_{k}$ and $r_{k}$ are independent white Gaussian
processes with zero mean and covariances $Q_{k},R_{k}$, respectively.
$x_{0}$ is assumed to have a zero-mean Gaussian distribution with
covariance $P_{0}$, independent of $q_{1},r_{0}$. For convenience, we will sometimes refer to $P_{0}$ as $Q_{0}$. The matrices $A_{k},C_{k},Q_{k}.R_{k}$
and $P_{0}$ are system parameters with appropriate dimensions, and
assumed to be known. Considering the MSE distortion, we denote
\begin{equation}
\hat{x}_{k|s}\triangleq\mathop{\mathrm{argmin}}_{\hat{x}}\mathbb{E}\left[\|x_{k}-\hat{x}\|^{2}|y_{0},\dots,y_{s}\right],
\end{equation}
namely the optimal MSE estimator of the state at time $k$, given measurements up to time $s$. Under the assumptions mentioned above, Kalman filters produce the mean state estimate $\hat{x}_{k|k}$, an MSE-optimal estimator of $x_{k}$ given the observations up to
time $k$. The \textit{Kalman optimal state} $\hat{x}_{k}^{*}\equiv\hat{x}_{k|k}$
is given by the recurrence 
\begin{equation}
\hat{x}_{k}^{*}=A_{k}\hat{x}_{k-1}^{*}+K_{k}\II_{k},
\end{equation}
where $K_{k}$ is the \textit{optimal Kalman gain} \citep{kalman1960filtering}, given explicitly in Algorithm \ref{alg:Kalman-Filter}.

The vector $\II_{k}$ is the \textit{innovation} process,
\begin{equation}
\II_{k}=y_{k}-C_{k}\hat{x}_{k|k-1},
\end{equation}
describing the contribution of the new observation $y_k$ over the optimal prediction based on previous observations.
Since we are in the Linear-Gaussian setup, we have that the innovation
state $\II_{k}$ is orthogonal to the measurements $y_{0},\dots,y_{k-1}$,
guaranteeing the MSE optimality of the estimation. The calculation
of Kalman state is summarized in Algorithm \ref{alg:Kalman-Filter}.

\begin{algorithm}[H]
\caption{\label{alg:Kalman-Filter}Kalman Filter}

\begin{algorithmic}
\STATE \textbf{initialize:} $\hat{x}_{0}^{*}=K_{0} y_{0}=P_{0}C_{0}^{\T}\Sigma_{Y_{0}}^{-1}y_{0},\quad P_{0|0}=P_{0}-P_{0}C_{0}^{\T}\Sigma_{Y_{0}}^{-1}C_{0}P_{0},\II_0 = y_0$, $S_0=C_0P_0C_0^\T + R_0$.

\FOR{$k=1$ to $T$ }

\STATE calculate prior: $\hat{x}_{k|k-1}=A_{k}\hat{x}_{k-1|k-1},\quad P_{k|k-1}=A_{k}P_{k-1|k-1}A_{k}^{\T}+Q_{k}$
\STATE calculate Innovation: $\II_{k}=y_{k}-C_{k}\hat{x}_{k|k-1}^{},\quad S_{k}=C_{k}P_{k|k-1}C_{k}^{\T}+R_{k}$
\STATE Kalman gain: $K_{k}=P_{k|k-1}C_{k}^{\T}S_{k}^{-1}$
\STATE update (posterior): $\hat{x}_{k}^{*}=\hat{x}_{k|k}=\hat{x}_{k|k-1}+K_{k}\II_{k},\quad P_{k|k}=(I-K_{k}C_{k})P_{k|k-1}$
\ENDFOR
\end{algorithmic}
\end{algorithm}

\newpage
\input{appendix/app_linearfilters.tex}
\newpage
\input{appendix/app_directmethod}
\newpage
\input{appendix/app_recurse}
\newpage
\input{appendix/app_generalizedSDP}

\newpage
\input{appendix/app_stationary}
\newpage
\input{appendix/app_notation}
\newpage
\input{appendix/pk_demo}

%% file: appendix/app_linearfilters.tex
\section{Optimality of linear filters (proof of Thm.~\ref{thm:generalFilter})}\label{app:optimality_lin}

In this section we show that under a family of optimality criteria \eqref{eq::cost_C:def} and perfect-perceptual quality and causality constraints (\ref{eq:causality}-\ref{eq:perfect_perception}), linear filters of the form \eqref{eq::pkl::Jk:innovatonUpsilon} are optimal.
We start with the following.

\begin{theorem}
\label{thm::pkl:causal jk independent of innovation}
Let $Y_{0}^{T}=(y_{0},\ldots,y_{T})$ be the set of measurements \eqref{eq:eq:kalman:setting2}, and
let $(J_{0}^{T},Y_{0}^{T})$ be a joint distribution s.t. $J_{k}$
is independent of $y_{k+n}$ given $Y_{0}^{k}$ for all $k\in[0,T]$
and $n\geq1$. Then, 
\begin{equation}
\EEb{J_{k}\II_{k+n}^{\T}}=0.
\end{equation}
\end{theorem}

\begin{proof}
Denote $\hat{J}_{k}=\EEb{J_{k}|\II_{0}^{k}}$. We can write the measurements
as a linear function of the innovations, $y_{k}=\sum_{t=0}^{k}H_{k,t}\II_{t}$.
We have 
\begin{equation}
\hat{y}_{k+n}^{k+n-1}\triangleq \EEb{y_{k+n}|Y_{0}^{k+n-1}}=\EEb{y_{k+n}|\II_{0}^{k+n-1}}=\sum_{t=0}^{k+n-1}H_{k+n,t}\II_{t},
\end{equation}
and
\begin{equation}
\hat{y}_{k+n}^{k}\triangleq \EEb{y_{k+n}|Y_{0}^{k}}=\EEb{y_{k+n}|\II_{0}^{k}}=\sum_{t=0}^{k}H_{k+n,t}\II_{t}=\EEb{\hat{y}_{k+n}^{k+n-1}|\II_{0}^{k}}.
\end{equation}

For any $k$ and $n=1$,  $\II_{k+1} = y_{k+1}-\hat{y}_{k+1}^k$, and therefore
\begin{equation}
\EEb{J_{k}\II_{k+1}^{\T}}
=\EEb{\EEb{J_{k}\left[y_{k+1}-\hat{y}_{k+1}^{k}\right]^{\T}|\II_{0}^{k}}}
=\EEb{\hat{J}_{k}\left[\hat{y}_{k+1}^{k}\right]^{\T}-\EEb{J_{k}|\II_{0}^{k}}\left[\hat{y}_{k+1}^{k}\right]^{\T}}=0.
\end{equation}
This is due to the facts that $J_{k}$ and $y_{k+1}$ are independent
given the condition, and that $\hat{y}_{k+1}^{k}$ is a deterministic
function of $\II_{0}^{k}$.

Now, assume we know that $\EEb{J_{k}\II_{t}^{\T}}=0$ for $k+1\leq t\leq k+n-1$.
We can write 
\begin{align}
\EEb{J_{k}\II_{k+n}^{\T}}
&=\EEb{\EEb{J_{k}\left[y_{k+n}-\hat{y}_{k+n}^{k+n-1}\right]^{\T}|\II_{0}^{k}}}
\nonumber\\&=\EEb{\hat{J}_{k}\left[\hat{y}_{k+n}^{k}\right]^{\T}-\EEb{J_{k}\sum_{t=0}^{k+n-1}\II_{t}^{\T}H_{k+n,t}^{\T}|\II_{0}^{k}}}
\nonumber\\&=\EEb{\hat{J}_{k}\left[\hat{y}_{k+n}^{k}\right]^{\T}-\EEb{J_{k}\sum_{t=0}^{k}\II_{t}^{\T}H_{k+n,t}^{\T}|\II_{0}^{k}}-\EEb{J_{k}\sum_{t=k+1}^{k+n-1}\II_{t}^{\T}H_{k+n,t}^{\T}|\II_{0}^{k}}}
\nonumber\\&=\EEb{\hat{J}_{k}\left[\hat{y}_{k+n}^{k}\right]^{\T}-\hat{J}_{k}\left[\hat{y}_{k+n}^{k}\right]^{\T}-\sum_{t=k+1}^{k+n-1}\EEb{J_{k}\II_{t}^{\T}|\II_{0}^{k}}H_{k+n,t}^{\T}}
\nonumber\\&=-\sum_{t=k+1}^{k+n-1}\EEb{J_{k}\II_{t}^{\T}}H_{k+n,t}^{\T}
\nonumber\\&=0.
\end{align}

\end{proof}

We now show that for every filter which is feasible under \eqref{eq:causality} and \eqref{eq:perfect_perception}, one can find a linear filter, jointly Gaussian with the measurement set, attaining the same cost objective.
\begin{theorem} \label{thm::pkl:joint gaussian equivalent}
Let $Y_{0}^{T}=(y_{0},\ldots,y_{T})$ be the set of measurements \eqref{eq:eq:kalman:setting2}, and let $\mathcal{J}_{0}^{T}=(\mathcal{J}_{0},\ldots,\mathcal{J}_{T})$ be jointly distributed with $Y_{0}^{T}$
such that:
\begin{enumerate}[(i)]
\item \label{enum:cond_i} $\mathcal{J}_{0}^{T}\sim \mathcal{N}\left(0,\mathrm{diag}\{P_{0},Q_{1},\ldots,Q_{T}\}\right)$.
\item \label{enum:cond_ii} $\mathcal{J}_{k}$ is independent of $y_{k+n}$ given $Y_{0}^{k}$
for all $k\in[0,T]$ and $n\geq1$. 
\item \label{enum:cond_iii} $\sum_{k=0}^{T}\w_{k}\EEb{\|x_{k}-\chi_{k}\|^{2}}=\C$, where $\chi_{k}$
is the process given by $\chi_{k}=A_{k}\chi_{k-1}+\mathcal{J}_{k}$ with
$\chi_{0}=\mathcal{J}_{0}$.
\end{enumerate}
Then, there exists a joint Gaussian distribution $(J_{0}^{T},Y_{0}^{T})$
in which \eqref{enum:cond_i} and \eqref{enum:cond_ii} hold, and the estimator given by 
\begin{equation}
    \hat{x}_k = A_k\hat{x}_{k-1}+J_k, \quad \hat{x}_0 = J_0
\end{equation}
achieves the same cost \eqref{enum:cond_iii}, namely $\sum_{k=0}^{T}\w_{k}\EEb{\|x_{k}-\hat{x}_{k}\|^{2}}=C$.

Furthermore, we can write
\begin{equation}
\label{appeq::J_k linear full form}
J_{k}=\pi_{k}\II_{k}+\phi_{k}\UII_{k}+w_{k},
\end{equation}
where 
\begin{equation}
\UII_{k}=\II_{0}^{k-1}-\EEb{\II_{0}^{k-1}|J_{0}^{k-1}}
\end{equation}
and $w_{k}$ is a white Gaussian noise, independent of $Y_{0}^{T}$
and $J_{0}^{k-1}$.
\end{theorem}

\begin{proof}
Let $(J_{0}^{T},Y_{0}^{T})$ be the Gaussian distribution defined
by the moments of $(\mathcal{J}_{0}^{T},Y_{0}^{T})$ up to second
order. We observe that from Theorem~\ref{thm::pkl:causal jk independent of innovation}
above, $J_{k}$ is independent
of all future innovations $\II_{k+n}$, namely it is based only on measurements up to time $k$. 
Using the notions of Theorem~\ref{thm::pkl:causal jk independent of innovation}'s proof,  
\begin{align}
    \EEb{(J_k-\hat{J}_k)(y_{k+n}-\hat{y}_{k+n}^k)^\T|Y_0^k} & = \EEb{(J_k-\hat{J}_k)\sum_{t=k+1}^{k+n}\II_t^\T H_{k+n,t}^\T|\II_0^k}
    \nonumber\\
    &= \sum_{t=k+1}^{k+n} \left[ \EEb{J_k \II_t^\T | \II_0^k}
    -\hat{J}_k\EEb{\II_t^\T|\II_0^k} \right] H_{k+n,t}^\T
\nonumber\\
    &=
    \sum_{t=k+1}^{k+n} \EEb{\EEb{J_k \II^\T_t|\II_{t},\II_0^k}|\II_0^k}H_{k+n,t}^\T
\nonumber\\
    &=
    \sum_{t=k+1}^{k+n} \EEb{\EEb{J_k |\II_{t},\II_0^k}\II^\T_t|\II_0^k}H_{k+n,t}^\T
\nonumber\\
    &=
    \sum_{t=k+1}^{k+n} \EEb{\EEb{J_k |\II_0^k}\II^\T_t|\II_0^k}H_{k+n,t}^\T
\nonumber\\
    &=
    \sum_{t=k+1}^{k+n} \EEb{J_k |\II_0^k}\EEb{\II^\T_t|\II_0^k}H_{k+n,t}^\T
\nonumber\\
    &=0.
\end{align}
This means that $J_k$ and $y_{k+n}$ are independent given $Y_0^k$, which proves \eqref{enum:cond_ii}.

From \eqref{eq:D_k::Lyapunov_general} we see that the cost functional depends only on the second order statistics of $(\mathcal{J}_{0}^{T},\II_{0}^{T})$ which are identical to those of $({J}_{0}^{T},\II_{0}^{T})$, hence \eqref{enum:cond_iii} holds:
\begin{equation}
\sum_{k=0}^{T}\w_{k}\EEb{\|x_{k}-\hat{x}_{k}\|^{2}}=\sum_{k=0}^{T}\w_{k}\EEb{\|x_{k}-\chi_{k}\|^{2}}=\C.
\end{equation}

To prove \eqref{appeq::J_k linear full form}, we now write
\begin{equation}
J_{k}=\varepsilon_{k}+w_{k},
\end{equation}
where $\varepsilon_{k}=\EEb{J_{k}|Y_{0}^{T},J_{0}^{k-1}}$, and $w_{k}=J_{k}-\EEb{J_{k}|Y_{0}^{T},J_{0}^{k-1}}$
is independent of $Y_{0}^{T}$ and $J_{0}^{k-1}$. Now, since both
$J_{k}$ and $J_{0}^{k-1}$ are independent of $\II_{k+1}^{T}$,
\begin{equation}
\label{eqn::pkl: epsk as linear sum}
\varepsilon_{k}
=\EEb{J_{k}|Y_{0}^{T},J_{0}^{k-1}}
=\EEb{J_{k}|\II_{0}^{k},J_{0}^{k-1}}
=\sum_{t=0}^{k}\phi_{k,t}\II_{t}+\sum_{t=0}^{k-1}\psi_{k,t}J_{t}.
\end{equation}
$J_{k}$ is independent of $J_{0}^{k-1}$, thus 
\begin{equation}
\EEb{J_{k}|J_{0}^{k-1}}
=\EEb{\EEb{J_{k}|\II_{0}^{T},J_{0}^{k-1}}|J_{0}^{k-1}}
=0.
\end{equation}
Conditioning both sides of \eqref{eqn::pkl: epsk as linear sum} on $J_{0}^{k-1}$ and taking expectations,
\begin{equation}
\label{eqn::pkl: epsk as linear sum after condition}
0=\sum_{t=0}^{k}\phi_{k,t}\EEb{\II_{t}|J_{0}^{k-1}}+\sum_{t=0}^{k-1}\psi_{k,t}J_{t}.
\end{equation}
Note that $\EEb{\II_{k}|J_{0}^{k-1}}=0$, which together with 
\eqref{eqn::pkl: epsk as linear sum after condition}
implies
\begin{equation}
\varepsilon_{k}
=\phi_{k,k}\II_{k}+\sum_{t=0}^{k-1}\phi_{k,t}\left[\II_{t}-\EEb{\II_{t}|J_{0}^{k-1}}\right]
=\pi_{k}\II_{k}+\phi_{k}\UII_{k}.
\end{equation}

Now, all we have left to show is that $w_{k}$ is a white sequence.
Since $w_{k+n}$ ($n\geq1$) is independent of $J_{0}^{k}$ and $\II_{0}^{T}$ (which also constitute $\UII_k$),
it is easy to obtain
\begin{equation}
\EEb{w_{k+n}w_{k}^{\T}}=\EEb{w_{k+n}\left[J_{k}-\pi_{k}\II_{k}-\phi_{k}\UII_{k}\right]^{\T}}
=0.
\end{equation}
\end{proof}

\begin{corollary}
Given a cost objective of the form $\C=\sum_{k=0}^{T}\w_{k}\EEb{\|x_{k}-\hat{x}_{k}\|^{2}}$,
there exists a \textit{linear} filter of the form
\begin{equation}
J_{k}=\pi_{k}\II_{k}+\phi_{k}\UII_{k}+w_{k},
\end{equation}
 such that 
\begin{align}
\hat{x}_{0}&=J_0
\\
\hat{x}_{k}&=A_{k}\hat{x}_{k-1}+J_{k}, \, k=1,\ldots,T
\end{align}
 is an optimal estimator under the perfect perceptual quality and
causality constraints (\ref{eq:causality}-\ref{eq:perfect_perception}). 
\end{corollary}

\begin{proof} Under the perfect perceptual quality constraint, an estimate sequence
$\chi_{k}$ must satisfy that
\begin{equation}
\mathcal{J}_{k}=\chi_{k}-A_{k}\chi_{k-1}
\end{equation}
 is a white Gaussian process with covariances $Q_{k}$. If, in addition, $\chi_{k}$ satisfies the causality condition \eqref{eq:causality}, so does
$\mathcal{J}_{k}$. We conclude from Theorem~\ref{thm::pkl:joint gaussian equivalent} that there
exists a causal linear filter $J_k$ that achieves the same
expected objective $\C$ as $\chi_k$.

Now, note again that from \eqref{eq:D_k::Lyapunov_general}, for perfect-perceptual quality causal filters, the objective $C$ is a continuous function of the covariance matrix \begin{equation}
    \EEb{\mathcal{J}_0^T \left(\II_0^T\right)^\top}
        =
        \begin{bmatrix}
        \mathrm{diag}\{P_{0},Q_{1},\ldots,Q_{T}\} & L \\ L^\top & \mathrm{diag}\{S_{0},S_{1},\ldots,S_{T}\}
        \end{bmatrix}
        \succeq 0,
    \end{equation}
    where, due to the causality demand, $L$ is a quasi lower triangular matrix. The set of such feasible matrices is non-empty, closed (since it is the intersection of the closed cone of PSD matrices with a finite set of hyperplanes) and bounded. Hence, $C$ attains a minimal value on some joint distribution $p_{\mathcal{J}_{0}^{T},Y_{0}^{T}}$, which can be chosen to be joint-Gaussian as we have seen.

\end{proof}

%% file: appendix/app_directmethod.tex
\section{A Direct optimization approach to perfect-perceptual quality filtering}
\label{app::direct approach}

For the sake of completeness, we now discuss a method for optimizing non-recursive perfect-perceptual quality filter coefficients. This approach leads to convex programs. However, as we will see next, it might become impractical for large configurations. 

Let $J=J_0^T \sim \ndist{0}{Q} $, where $Q=\diag{\{Q_k\}_{k=0}^T}$, be a causal function of the measurements, $J=\varPhi \II + W$, where $\II=\II_0^T$ is the innovation process with covariance $S=\diag{S_k}$ and $W$ is an independent noise. Now, $\hat{X}=\hat{X}_0^T=A_J J$ is the filter's output, where
\begin{equation}
A_J = \begin{bmatrix}
I & 0 & \dots & 0 \\
A_1 & I & \dots & 0  \\
\vdots & \vdots & \ddots & \vdots
\\
\prod_{k=0}^{T-1}A_{T-k} & \prod_{k=1}^{T-1}A_{T-k} & \dots & I  
\end{bmatrix}.
\end{equation}

Recall $X^*$ is the Kalman filter output given by $X^* = A_{J}K\II$, where $K=\diag{K_k}$. Let
$\W=\diag{\w_k}\otimes I
_{n_x}$ be a weighting matrix. 
The objective \eqref{eq::cost_C:def} is now given by
\begin{align}
\C(\hat{X})&=\EEb{(\hat{X}-X^*)^\T \W (\hat{X}-X^*)} \nonumber\\ &= \tr{\W \EEb{\hat{X}\hat{X}^T} + \W \EEb{X^*X^{*\T}} -2\W\EEb{\hat{X}X^{*\T}} }.
\end{align}
 Hence, we have to maximize 
\begin{align}
\C(\varPhi)&= 2\tr{\W\EEb{\hat{X}X^{*\T}} } \nonumber\\&=2\tr{\W A_J\varPhi S K^\T A_{J}^\T} \nonumber\\& =2\tr{ (\varPhi S) K^\T A_{J}^\T \W A_J} \nonumber\\&=2\tr{ \varPhi S K^\T B},
\end{align}
where $B\triangleq A_{J}^\T \W A_J$. This is subject to the perfect perceptual-quality constraint
\begin{equation}
Q-\varPhi S \varPhi^\T \succeq 0, \ \text{or equivalently} \ \begin{bmatrix}
Q & \varPhi S \\
S \varPhi^\T & S
\end{bmatrix} \succeq 0,
\end{equation}
where $\varPhi$ is a lower quasi-triangular matrix (causality constraint)
\begin{equation}
\varPhi = \begin{bmatrix}
\varPhi_{0,0} & 0 & \dots & 0 \\
\varPhi_{1,0} & \varPhi_{1,1} & \dots & 0  \\
\vdots & \vdots & \ddots & \vdots
\\
\varPhi_{T,0} & \varPhi_{T,1} & \dots & \varPhi_{T,T}  
\end{bmatrix}.
\end{equation}

Again, under this formulation, 
\begin{equation}
J = \varPhi \II +W,
\end{equation}
where $W \sim \ndist{0}{Q-\varPhi S \varPhi^\T}$ is a Gaussian noise independent of $\II$. Note that $W_0^T$ might not be a white sequence in this case, since its covariance might not be a block-diagonal matrix. As a result, the noise sequence has to be sampled dependently. 
Also note that this problem possesses the same memory complexity as
\eqref{eq::pkl::Jk:innovatonUpsilon}. 
To conclude, this method leads to convex, but large optimization programs, and is impractical for high dimensional settings or long temporal sequences.

%% file: appendix/app_recurse.tex
\section{Derivation of recursive perfect-perceptual quality filters}\label{app:filter_deriv}

We now derive the recursive expression \eqref{eq:pkf::recursiveUpsSigma}-\eqref{eq:pkf::recursivePsi} for the filter given in \eqref{eq::pkl:J_k::fullJk},
\begin{align}
\hat{x}_{k}&=A_{k}\hat{x}_{k-1}+J_k,
\\
J_k &= \varPhi_{k}A_{k}\Upsilon_{k}+\Pi_{k}K_{k}\II_{k}+w_{k},\,w_{k}\sim\mathcal{N}\left(0,\Sigma_{w_{k}}\right),
\end{align}
 defined by the coefficients $\left\{ \Pi_{k},\varPhi_{t}\right\} _{t=0}^{T}$ fulfilling the constraints \eqref{eq::pkl:Sigma_w_full}.
Recall
\begin{equation}
\Upsilon_{k}\triangleq
\hat{x}_{k-1}^{*}-\EEb{\hat{x}_{k-1}^{*}|\hat{x}_{0},\ldots,\hat{x}_{k-1}}
=\hat{x}_{k-1}^{*}-\EEb{\hat{x}_{k-1}^{*}|{J}_{0},\ldots,{J}_{k-1}}
\end{equation}
where $\hat{x}_{k}^{*}$ is the Kalman state. ${J}_{0}^{k-1},\varUpsilon_{k},\II_k,w_k$ are jointly-Gaussian and independent, and we have
\begin{align}
    \EEb{J_n J_k^\T} &= Q_k \delta_{n=k},
    \\
    \EEb{\II_k J_k^\T} &= S_k K_k^\T \Pi_k^\T,
    \\
    \EEb{\varUpsilon_k J_k^\T} &= \Sigma_{\varUpsilon_k} A_k^\T \varPhi_k^\T.
\end{align}
We can write
\begin{align}
    \varUpsilon_{k+1}-A_{k}\varUpsilon_{k} & = \hat{x}^*_{k}-A_{k}\hat{x}^*_{k-1} - \left[ \EEb{\hat{x}^*_{k}|J_0^k}-A_{k}\EEb{\hat{x}^*_{k-1}|J_0^{k-1}} \right]
    \nonumber\\
    &=K_k\II_k - K_k \EEb{\II_k|J_0^k} - A_{k}
    \left[ \EEb{\hat{x}^*_{k-1}|J_0^k}-\EEb{\hat{x}^*_{k-1}|J_0^{k-1}} \right]
\end{align}
Since $J_0^k$ is an independent sequence, and since $\II_k$ depends only on $J_k$,
\begin{equation}
  K_k\EEb{\II_k|J_0^k} 
  = K_k\EEb{\II_k|J_k}
  = K_k S_k K_k^\T \Pi_k^\T  Q_k^\dagger J_k.  
\end{equation}
We also have that $\varUpsilon_k,J_k$ are independent of $J_0^{k-1}$, implying
\begin{align}
    \EEb{\hat{x}^*_{k-1}|J_0^k}-\EEb{\hat{x}^*_{k-1}|J_0^{k-1}} 
    &= \EEb{\hat{x}^*_{k-1}-\EEb{\hat{x}^*_{k-1}|J_0^{k-1}}|J_0^k}
    \nonumber\\&=
    \EEb{\varUpsilon_k|J_0^k} =
    \EEb{\varUpsilon_k|J_k}
    \nonumber\\
    &=
    \Sigma_{\varUpsilon_k} A_k^\T \varPhi_k^\T Q_k^\dagger J_k. 
\end{align}
Hence,
\begin{equation}
    \varUpsilon_{k+1} = A_k\varUpsilon_{k} +K_k \II_k -  \varPsi_k Q_k^\dagger J_k, 
\end{equation}
where we denote
\begin{equation}
\varPsi_k \triangleq M_k \Pi_k^\top +A_k \Sigma_{\varUpsilon_k} A_k^\T \varPhi_k^\top. 
\end{equation}

The covariance is then given by the recursive form
\begin{align}
\Sigma_{\varUpsilon_{k+1}} &= A_k \Sigma_{\varUpsilon_k} A_k^\top +M_k +\varPsi_k      Q_k^\dagger \varPsi_k^\T 
\nonumber\\
& - A_k \Sigma_{\varUpsilon_k} A_k^\T \varPhi_k^\T Q_k^\dagger \varPsi_k^\T
-
K_k S_k K_k^\T \Pi_k^\T Q_k^\dagger \varPsi_k^\T
\\
& - \left[ A_k \Sigma_{\varUpsilon_k} A_k^\T\varPhi_k^\T Q_k^\dagger \varPsi_k^\T \right]^\T
-
\left[ K_k S_k K_k^\T \Pi_k^\T Q_k^\dagger \varPsi_k^\T \right]^\T
\\
&= A_k \Sigma_{\varUpsilon_k} A_k^\top +M_k -\varPsi_k      Q_k^\dagger \varPsi_k^\T.
\end{align}

At time $k=0$ we have $\varUpsilon_0=0$ and $\Sigma_{\varUpsilon_0}=0$.

\begin{remark}[The non-reduced case]
For the full, non-reduced linear filter \eqref{eq::pkl::Jk:innovatonUpsilon}-    \eqref{eq::pkl:uii_k}
, we have the following similar formula
\begin{equation}
    \UII_k = \begin{bmatrix}
        \II_{k-1} \\ \UII_{k-1}
    \end{bmatrix}
    - 
    \begin{bmatrix}
        S_{k-1} & 0 \\ 0 & \Sigma_{\UII_{k-1}}
    \end{bmatrix}
    \begin{bmatrix}
        \pi_{k-1}^\T \\ \phi_{k-1}^\T
    \end{bmatrix}
    Q_{k-1}^\dagger J_{k-1} 
\end{equation}
and
\begin{equation}
    \Sigma_{\UII_k} = \begin{bmatrix}
        S_{k-1} & 0 \\ 0 & \Sigma_{\UII_{k-1}}
    \end{bmatrix}
    - 
    \begin{bmatrix}
        S_{k-1} & 0 \\ 0 & \Sigma_{\UII_{k-1}}
    \end{bmatrix}
    \begin{bmatrix}
        \pi_{k-1}^\T \\ \phi_{k-1}^\T
    \end{bmatrix}
    Q_{k-1}^\dagger 
    \begin{bmatrix}
        \pi_{k-1}^\T \\ \phi_{k-1}^\T
    \end{bmatrix}^\T
    \begin{bmatrix}
        S_{k-1} & 0 \\ 0 & \Sigma_{\UII_{k-1}}
    \end{bmatrix}.
\end{equation}
Notice, however, that the dimension of $\UII_k$ grows with time $k$.
\end{remark}

%% file: appendix/app_generalizedSDP.tex
\newpage
\section{A Generalized extremal problem with semidefinite constraints (proof of Thm.~\ref{Thm:pkf:optimal_perc_gain})}
\label{app::extremalSDP}

In this section we prove Theorem~\ref{Thm:pkf:optimal_perc_gain}. We start with a brief overview of the extremal problem of finding a minimal distance between distributions, and of general semi-definite programs.

To prove the Theorem we observe that \eqref{eq:stepk:traceproblem}, 
is a generalization of the extremal problem, and suggest a non-trivial dual form where, under our assumptions, strong duality holds.

\subsection{Minimal distance between distributions}

Consider two Gaussian distributions on $\mathbb{R}^{n}$ with
zero means and PSD covariance matrices $\Sigma_{1},\Sigma_{2}$ respectively.
We consider the problem of constructing a Gaussian vector $\left[X,Y\right]$
minimizing $\mathbb{E}\|X-Y\|^{2}$ while inducing the given marginal
distributions, 
$X\sim\mathcal{N}\left(0,\Sigma_{1}\right),Y\sim\mathcal{N}\left(0,\Sigma_{2}\right)
$. 
This problem is equivalent to the following maximization of correlation \citep{olkin1982distance}
\begin{align}
\tr{ 2\Pi}  & \rightarrow\max_\Pi,\quad \st \,\Sigma=\left[\begin{array}{cc}
\Sigma_{1} & \Pi\\
\Pi^\T & \Sigma_{2}
\end{array}\right]\succeq0.\label{eq:Olkin::maxCorreProblem}
\end{align}
We have the following results of \citet{olkin1982distance}.
\begin{lemma} \label{Olkin-1982,-Lemma}
\citep[Lemma 1]{olkin1982distance}. Let $\Sigma_{2}^{g}$ be any generalized inverse of $\Sigma_{2}$. Then $\Sigma\succeq 0$ iff 
\begin{equation}
\Sigma_{2}\Sigma_{2}^{g}\Pi^{\T}=\Pi^{\T}\label{eq:olkin_lem:cond1}
\, \mathrm{and} \, 
\Sigma_{1}-\Pi\Sigma_{2}^{g}\Pi^{\T}\succeq0.
\end{equation}
\end{lemma}

\begin{theorem}
\label{thm:=00005BOlkin-1982,-Thm.4=00005D}\citep[Thm. 4]{olkin1982distance}.
If $\im{\Sigma_{2}}\subseteq \im{\Sigma_{1}}$, then an optimal
solution to (\ref{eq:Olkin::maxCorreProblem}) is given by
\begin{equation}
\max_\Pi \tr{ 2\Pi} =2\tr{ \left(\Sigma_{2}^{1/2}\Sigma_{1}\Sigma_{2}^{1/2}\right)^{1/2}} ,\label{eq:olkin_thm:optimal_value}
\end{equation}
achieved by the argument
\begin{equation}
\Pi^{*}=\Sigma_{1}\Sigma_{2}^{1/2}\left[\left(\Sigma_{2}^{1/2}\Sigma_{1}\Sigma_{2}^{1/2}\right)^{1/2}\right]^{g}\Sigma_{2}^{1/2}.\label{eq:olkin_thm:optimal_correlation}
\end{equation}
In the case where $\im{\Sigma_{2}}=\im{\Sigma_{1}}$, $\Pi^{*}$
is a unique optimal argument.
\end{theorem}

Under the setting discussed in Sec.~\ref{sec::prelim:DP}, Theorem~\ref{thm:=00005BOlkin-1982,-Thm.4=00005D}
implies that in the more general case where $\Sigma_x \succeq 0$, the MSE-optimal perfect perceptual-quality estimator \eqref{eq::Gaussians:optimalT} is obtained by 
\begin{equation} \hat{x}=\TT^*x^*+w, \, \label{appeq::Gaussians:optimalT}
\quad {\TT}^{*}\triangleq \Sigma_{x}\Sigma_{x^*}^{\hlf}(\Sigma_{x^*}^{\hlf}\Sigma_{x}\Sigma_{x^*}^{\hlf})^{\hlf\dagger}\Sigma_{x^*}^{\hlf\dagger}.
\end{equation}
Here again, $w$ is a zero-mean Gaussian noise with covariance 
$\Sigma_{w}=\Sigma_{x}-\TT^{*}\Sigma_{x^*}^{}\TT^{*\T}$,
independent of $y$ and $x$, and $\Sigma_{x^*}^{\dagger}$ is the Moore-Penrose inverse of $\Sigma_{x^*}$.

\subsection{SDP Setting and duality - background}

\textit{Semi-definite programming} (SDP) \citep{freund2004introduction,vandenberghe1996semidefinite} is an optimization problem in
$X\in\mathbb{R}^{n\times n}$ of the form
\begin{align}
C\bullet X & \rightarrow \max_{X}\label{eq:SDP}\\
\st \ & A_{i}\bullet X=b_{i}\,,i=1,\ldots,m,\\
 & X\succeq0.
\end{align}
Here, $C,A_{i}$ are real symmetric matrices of appropriate dimensions,
and $A\bullet X=\mathrm{Tr}\{A^{\T}X\}$ is the Frobenius product. SDPs yield the Lagrangian 
\begin{align}
L(X,\lambda,\nu)
 &= \nu^{\T}b+\left(C-\sum_{i=0}^{m}\nu_{i}A_{i}\right)\bullet X+\lambda\rho_{min}(X) \nonumber\\
&= \nu^{\T}b+\left(C-\sum_{i=0}^{m}\nu_{i}A_{i}\right)\bullet X+\min_{Y\succeq0,\mathrm{Tr}Y=\lambda}Y\bullet X,
\end{align}
where $\lambda\geq0$ and $\rho_{min}$ is the minimal eigenvalue. The Dual
problem (DSP) is given by
\begin{equation}
\nu^{\T}b\rightarrow\min_\nu,\quad\st\,C-\sum_{i=0}^{m}\nu_{i}A_{i}\preceq0.\label{eq:SDP::Dual}
\end{equation}
In this case, strong duality exists iff the SDP is strictly feasible, \textit{i.e.} it
has a feasible solution interior to the feasible set, $X\succ0$. This condition is sometimes referred to as the \textit{Slater condition}.

\subsection{A generalized extremal problem with strong duality} \label{sec:Ageneralizedextremal}

Recall $Q_k,M_k,B_k$ are real, symmetric positive semidefinite $n\times n$ matrices,
and the optimization problem \eqref{eq:stepk:traceproblem}, 
\begin{equation}
\label{eq:stepk:traceproblem_APP}
\tr{\tilde{\Pi}_{k}M_{k}B_{k}}=\tr{\tilde{\Pi}_{k}M_{k}M_{k}^{\dagger}M_{k}B_{k}}\rightarrow\max_{\tilde{\Pi}_{k}},\,\st\,Q_{k}-\tilde{\Pi}_{k}M_{k}\tilde{\Pi}_{k}^{\T}\succeq0.
\end{equation}
Since \eqref{eq:stepk:traceproblem_APP} involves a single time step, we will omit the index $k$. 

We consider $\Pi=\tilde{\Pi}M$, hence $\Pi^{\T}=MM^{\dagger}\Pi^{\T}$,
and since $M=MM^{\dagger}M$ we rewrite \eqref{eq:stepk:traceproblem_APP} as
\begin{equation}
\tr{\Pi B}=\tr{B\Pi}\rightarrow\max_\Pi,\,\st,\,Q-\Pi M^{\dagger}\Pi^{\T}\succeq0,\,\Pi^{\T}=MM^{\dagger}\Pi^{\T}.\label{eq:extrem_prob}
\end{equation}
By Lemma \ref{Olkin-1982,-Lemma}, the constraints in \eqref{eq:extrem_prob} are equivalent
to 
\begin{equation}
X\triangleq\begin{bmatrix}\begin{array}{cc}
Q & \Pi\\
\Pi^\T & M
\end{array}\end{bmatrix}\succeq0.\label{eq:X_PSD}
\end{equation}
 This can be formulated as the semi-definite program, 
\begin{equation}
C\bullet X\rightarrow\max_X,\,\st\,\begin{cases}
A_{ij}^{Q}\bullet X & =Q_{ij}\\
A_{ij}^{M}\bullet X & =M_{ij}
\end{cases},0\leq i\leq j\leq n-1,\,X\succeq0,\label{eq:extem_prob:SDPform}
\end{equation}
where $C=\hlf \left[\begin{array}{cc}
0 & B\\
B & 0
\end{array}\right]$, and $A_{ij}^{Q}=\begin{bmatrix}\begin{array}{cc}
\Lambda_{ij} & 0\\
0 & 0
\end{array}\end{bmatrix},A_{ij}^{M}=\begin{bmatrix}\begin{array}{cc}
0 & 0\\
0 & \Lambda_{ij}
\end{array}\end{bmatrix}$, $\Lambda_{ij}=\frac{1}{2}(e_{ij}+e_{ji})$. 

Note that when $B$ is a scalar
matrix, \eqref{eq:extrem_prob} is similar to the problem studied in \citet{olkin1982distance}. Their approach was later extended by \citet{shapiro1985extremal} to general linear objectives, where the Slater condition holds.

\subsubsection{Strong duality}

The SDP \eqref{eq:extem_prob:SDPform} yields the standard dual formulation
\begin{equation}
Q\bullet\nu_{Q}+M\bullet\nu_{M}\rightarrow\min_{\nu_{Q},\nu_{M}},\,\st\,\begin{bmatrix}\begin{array}{cc}
\nu_{Q} & -\hlf B\\
-\hlf B & \nu_{M}
\end{array}\end{bmatrix}\succeq0,\nu_{Q},\nu_{M}\in\mathbb{R}^{n_{x}\times n_{x}}.\label{eq:extrem_prob:DPform}
\end{equation}
This should give us a hint about the optimal solution to \eqref{eq:extrem_prob}.
Pay attention, however, that according to the theory, strong duality in
\eqref{eq:extrem_prob:DPform} is guaranteed only if $Q,M\succ0$,
which might not be the case (see \textit{e.g.} \citep{shapiro1985extremal}). To get a tight bound for the general case
$Q,M\succeq0$, we now provide an alternative form of duality to
\eqref{eq:extrem_prob}.

The following is an adaptation of techniques used in \citet{olkin1982distance}. We start with the following Lemma. 
\begin{lemma}
\label{lem:duality}Let $\Pi$ be a feasible solution to (\ref{eq:extrem_prob}),
$R,G\in\mathbb{R}^{n_{x}\times n_{x}}$ are general matrices. Then,
\begin{equation}
\tr{QRR^{\T}+BMBGG^{\T}}\geq2\tr{\Pi BGR^{\T}}.
\end{equation}
\end{lemma}
\begin{proof}
From the non-negativity of $X$ in \eqref{eq:X_PSD} we have
\begin{equation}
\begin{array}{l}
\left[R^{\T},-G^{\T}B\right]\begin{bmatrix}\begin{array}{cc}
Q & \Pi\\
\Pi^\T & M
\end{array}\end{bmatrix}\left[\begin{array}{c}
R\\
-BG
\end{array}\right]=\\
R^{\T}QR+G^{\T}BMBG-R^{\T}\Pi BG-G^{\T}B\Pi^{\T}R\succeq0.
\end{array}
\end{equation}
The trace is nonegative, hence we have the desired result.
\end{proof}
\begin{remark}
Similarly, we can obtain
\begin{equation}
\tr{QBRR^{\T}B+MGG^{\T}}\geq2\tr{B\Pi GR^{\T}}.
\end{equation}
\end{remark}
Now, we suggest an alternative to (DSP) \eqref{eq:extrem_prob:DPform},
where strong duality will hold.
\begin{theorem}
\label{thm::Strong-duality}{[}Strong duality{]}. Let
\begin{align}
\Omega & =\left\{\Pi\in\mathbb{R}^{n_{x}\times n_{x}}:Q-\Pi M^{\dagger}\Pi^{\T}\succeq0,\Pi^{\T}=MM^{\dagger}\Pi^{\T}\right\},\\
\mathcal{S}= & \left\{\left(S,S^{-}\right):S,
S^{-}\succeq 0,
SS^{-}S=S,S^{-}SS^{-}=S^{-},\,BM=SS^{-}BM\right\},
\end{align}
and denote $M_{B}\triangleq BMB$. Assume $\im{M_{B}}\subseteq \im{Q}$.
Then,
\begin{equation}
\begin{array}{c}
\min_{\left(S,S^{-}\right)\in\mathcal{S}}\left\{ Q\bullet S+M\bullet\left(BS^{-}B\right)\right\} =\max_{\Pi\in\Omega}\tr{2\Pi B} \\ =2\tr{ \left(M_{B}^{1/2}QM_{B}^{1/2}\right)^{1/2} }.
\end{array}
\end{equation}
The extreme value is obtained for
\begin{align}
S^{*}&=M_{B}^{1/2}\left(M_{B}^{1/2}QM_{B}^{1/2}\right)^{1/2\dagger}M_{B}^{1/2},\\
S^{-*}&=M_{B}^{1/2\dagger}\left(M_{B}^{1/2}QM_{B}^{1/2}\right)^{1/2}M_{B}^{1/2\dagger},\\
\Pi^{*}&=QS^{*}M_{B}^{\dagger}BM=QM_{B}^{1/2}\left(M_{B}^{1/2}QM_{B}^{1/2}\right)^{1/2\dagger}M_{B}^{1/2\dagger}BM.
\end{align}
Optimal solution $\Pi^{*}$ is generally not unique.
\end{theorem}
To prove strong duality, we will use the following lemmas.
\begin{lemma}
\label{lem:images_eq}Assume PSD matrices $Q,M_{B}$ such that $\im{M_{B}}\subseteq \im{Q}$,
then $\im{M_{B}}=\im{M_{B}^{1/2}QM_{B}^{1/2}}$.
\end{lemma}
\begin{proof}
Recall $M_{B},M_{B}^{1/2}QM_{B}^{1/2}$ are real symmetric matrices. 

Let $v\in \Ker\{M_{B}^{1/2}QM_{B}^{1/2}\},$ we have $\|Q^{1/2}M_{B}^{1/2}v\|=0$
hence $M_{B}^{1/2}v\in \Ker\{Q^{1/2}\}\subseteq \Ker\{M_{B}^{1/2}\}$,
which yields $M_{B}v=0,$ implying $\Ker\{M_{B}^{1/2}QM_{B}^{1/2}\}\subseteq \Ker\{M_{B}\}$.
Opposite relation is trivial. 

We have
\begin{equation}
\im{M_{B}}=\Ker\{M_{B}\}^{\perp}=\Ker\{M_{B}^{1/2}QM_{B}^{1/2}\}^{\perp}=\im{M_{B}^{1/2}QM_{B}^{1/2}}.
\end{equation}
\end{proof}
\begin{lemma}
\label{lem:images_ineq}$\im{BM}\subseteq \im{BMB}.$
\end{lemma}
\begin{proof}
Let $v\in \Ker\{BMB\}$, then $\|M^{1/2}Bv\|=0$ and $Bv\in \Ker\{M^{1/2}\}=Ker\{M\}$.
Hence $\Ker\{BMB\}\subseteq \Ker\{MB\}$ . We have
\begin{equation}
\im{BM}=\Ker\{MB\}^{\perp}\subseteq \Ker\{BMB\}^{\perp}=\im{BMB}.
\end{equation}
\end{proof}
We are now ready to prove Theorem \ref{thm::Strong-duality}.
\begin{proof}
{[}Theorem \ref{thm::Strong-duality}{]}. Let $\Pi\in\Omega,$ then
$X\succeq0$ in (\ref{eq:X_PSD}). For any $\left(S,S^{-}\right)\in\mathcal{S}$
we can choose $R=S^{^{1/2}},G=S^{-}R$ . From the result of Lemma
\ref{lem:duality} it follows that
\begin{equation}
\begin{array}{c}
Q\bullet S+M\bullet\left(BS^{-}B\right)=\tr{QRR^{\T}+BMBGG^{\T}} \\
\geq 2\tr{\Pi BGR^{\T}}=2\tr{\Pi BS^{-}S}=2\tr{\Pi B}.
\end{array}
\end{equation}
The last equality holds since $BM=SS^{-}BM$, and $\im{\Pi^{T}}\subseteq \im{M}$.

We now prove that $\Pi^{*}\in\Omega$.
\begin{equation}
\begin{array}{ll}
Q-\Pi^{*}M^{\dagger}\Pi^{*\T} &\\ = Q - QM_{B}^{1/2}\left(M_{B}^{1/2}QM_{B}^{1/2}\right)^{1/2\dagger}M_{B}^{\dagger1/2}BMM^{\dagger}MBM_{B}^{\dagger1/2}\left(M_{B}^{1/2}QM_{B}^{1/2}\right)^{1/2\dagger}M_{B}^{1/2}Q\\
=  Q-QM_{B}^{1/2}\left(M_{B}^{1/2}QM_{B}^{1/2}\right)^{\dagger}M_{B}^{1/2}Q\\
= Q^{1/2}\left[I-Q^{1/2}M_{B}^{1/2}\left(M_{B}^{1/2}QM_{B}^{1/2}\right)^{\dagger}M_{B}^{1/2}Q^{1/2}\right]Q^{1/2}\\
= Q^{1/2}\left[I-Q^{1/2}M_{B}^{1/2}\left(M_{B}^{1/2}QM_{B}^{1/2}\right)^{\dagger}M_{B}^{1/2}Q^{1/2}\right]^{2}Q^{1/2}\succeq0.
\end{array}
\end{equation}
The last equality holds since it is easy to see that $\left[I-Q^{1/2}M_{B}^{1/2}\left(M_{B}^{1/2}QM_{B}^{1/2}\right)^{\dagger}M_{B}^{1/2}Q^{1/2}\right]$
is a symmetric (orthogonal) projection. 

We further prove that $S^{*},S^{-*}\in\mathcal{S}$. It is easy to
show that $S^{*},S^{-*}$ are symmetric generalized inverses, reflexive
to each other ($S^{-*}$ is in fact the Moore-Penrose inverse of $S^{*}$):
\begin{align}
S^{*}S^{-*} & =M_{B}^{1/2}\left(M_{B}^{1/2}QM_{B}^{1/2}\right)^{1/2\dagger}M_{B}^{1/2}M_{B}^{1/2\dagger}\left(M_{B}^{1/2}QM_{B}^{1/2}\right)^{1/2}M_{B}^{1/2\dagger}\\
 & =M_{B}^{1/2}\left(M_{B}^{1/2}QM_{B}^{1/2}\right)^{1/2\dagger}\left(M_{B}^{1/2}QM_{B}^{1/2}\right)^{1/2}M_{B}^{1/2\dagger}\\
 & =M_{B}^{1/2}M_{B}^{1/2\dagger}=M_{B}^{1/2\dagger}M_{B}^{1/2}\\
 & =S^{-*}S^{*}.
\end{align}
The equalities hold since by Lemma \ref{lem:images_eq}, 
\begin{equation}
\im{M_{B}^{1/2}}=Im\{M_{B}\}=\im{M_{B}^{1/2}QM_{B}^{1/2}}=\im{ \left(M_{B}^{1/2}QM_{B}^{1/2}\right)^{1/2}},
\end{equation}
 and since for a PSD matrix $R$, $RR^{\dagger}=R^{\dagger}R$ is
an orthogonal projection onto its image. Using Lemma \ref{lem:images_ineq}
we have 
\begin{equation}
S^{*}S^{-*}BM=M_{B}^{1/2\dagger}M_{B}^{1/2}BM=BM.
\end{equation}

It is now easy to verify that
\begin{equation}
Q\bullet S^{*}+M\bullet\left(BS^{-*}B\right)=2\tr{\Pi^{*}B}=2\tr{ \left(M_{B}^{1/2}QM_{B}^{1/2}\right)^{1/2} },
\end{equation}
which completes the proof.
\end{proof}

\begin{corollary}
    Under the  assumption,
$\im{M_{B}}\subseteq \im{Q_k}$,
the optimal gain in \eqref{eq:stepk:traceproblem}
 is given by
\begin{equation}
\Pi_k^{*}=Q_kM_{B}^{1/2}\left(M_{B}^{1/2}Q_kM_{B}^{1/2}\right)^{1/2\dagger}M_{B}^{\dagger1/2}B_k.\label{appeq:stepk:optimalPercepGain}
\end{equation}

\end{corollary}

\begin{remark}
Under the alternative assumption,
$\im{M_{k}}\subseteq \im{Q_k}$,
the optimal gain in \eqref{eq:stepk:traceproblem}
 is given by
\begin{equation}
\Pi_k^{*}=Q_k \bb M_{b}^{1/2}\left(M_{b}^{1/2}Q_bM_{b}^{1/2}\right)^{1/2\dagger}M_{b}^{\dagger1/2}\bb,
\label{eq:stepk:optimalPercepGain:alternative}
\end{equation}
where $\bb = B_k^{1/2},\ Q_b=\bb Q_k \bb,\ M_b=\bb M_k \bb$.
\end{remark}

\begin{proof}
Recall our goal in \eqref{eq:stepk:traceproblem} is to maximize $\tr{\Pi M B}=\tr{\bb \Pi M \bb}$ under the condition $Q-\Pi M \Pi^\top \succeq 0$ (we omit the index $k$). This is equivalent to minimizing $\EEb{\|\bb X-\bb Y\|^2}$ \emph{w.r.t} $\Pi$, where $(X,Y) \sim \ndist{0}{\Sigma}$ and $\Sigma = \begin{bmatrix} Q & \Pi M \\ M\Pi^\top & M\end{bmatrix}\succeq 0$. 

In this case, $(\bb X,\bb Y) \sim \ndist{0}{\Sigma_b}$ where $\Sigma_b = \begin{bmatrix} \bb Q \bb & \bb \Pi M \bb \\ \bb M\Pi^\top \bb & \bb M \bb\end{bmatrix}$. According to Thm.~\ref{thm:=00005BOlkin-1982,-Thm.4=00005D}, under the assumption $\im{\bb M \bb} \subseteq \im{\bb Q \bb}$, the minimal distance is achieved when 
\begin{equation}
\bb \Pi M \bb = \bb Q \bb M_b^\hlf  \left( M_b^\hlf Q_b M_b^\hlf \right)^{\hlf \dagger} M_b^\hlf.
\end{equation}
Note that $\im{M} \subseteq \im{Q}$ implies $\im{\bb M \bb} \subseteq \im{\bb Q \bb}$, and it is straightforward to verify that $Q-\Pi^*M\Pi^{*\top} \succeq 0$.
\end{proof}

%% file: appendix/app_stationary.tex
\section{Stationary settings}
\label{app::stationary}

A note is in place regarding the stationary perceptual Kalman filter. In
the Kalman steady-state regime, where dynamics \eqref{appeq:kalman:setting1}
-\eqref{appeq:eq:kalman:setting2}
 are time-invariant
and $T\rightarrow\infty$, the matrices $K$ and $S$ in Algorithm
\ref{alg:Kalman-Filter} are determined by the covariance matrix $P$,
\begin{align}
K=PC^{\T}(CPC^{\T}+R)^{-1},\quad
S=CPC^{\T}+R.
\end{align}
Here,  $C$ stands for the time-invariant observation matrix ($y_k=C x_k +r_k$) and $P$ is a solution to the Discrete-Time Algebraic Riccati equation
(DARE)
\begin{equation}
P=APA^{\T}-APC^{\T}(CPC^{\T}+R)^{-1}CPA^{\T}+Q.
\end{equation}
Similarly, under the steady-state regime, (\ref{eq:Agen:traceproblem})
becomes 
\begin{equation}
\begin{cases}
\tr{D} & \rightarrow \min_{\Pi}\\
\st & D=ADA^{\T}+Q+M-\Pi M-M\Pi^{\T}, M=KSK^{\T},\,Q-\Pi M\Pi^{\T}\succeq0
\end{cases}
\end{equation}
where $D$ obeys an (Algebraic) Lyapunov equation. 
If $A$ is stable, 
\begin{equation}
D(\Pi)=\sum_{k=0}^{\infty}A^{k}(Q+M-\Pi M-M\Pi^{\T})\left(A^{k}\right)^{\T}. \label{eq::pkl:stationaryDmatrix}
\end{equation}
Hence, stationary perceptual filter is of the form
\begin{align}
\hat{x}_k &= A\hat{x}_{k-1} + J_k, \label{eq::pkl:stationaryfilter} \\
J_{k}&=\Pi K\II_{k}+w_{k}, 
 \\ w_{k}&\sim\mathcal{N}\left(0,Q-\Pi M\Pi^{\T}\right),
\end{align}
and in order to find optimal gain $\Pi$, minimizing $\tr{D(\Pi)}$,  we have to solve
\begin{equation} \label{eq::pkl:stationaryMaxOptimization}
\max_{\Pi}\tr{\Pi MB} \, \st \, Q-\Pi M\Pi^{\T}\succeq0,
\end{equation}
where we define 
$
B\triangleq\sum_{k=0}^{\infty}\left(A^{k}\right)^{\T}A^{k}
$, 
and the solution (under the assumption $\im{B M B}\subseteq \im{Q}$) is given again by \eqref{eq:stepk:optimalPercepGain}.

%% file: appendix/app_notation.tex
\section{List of notations}
\label{app::notation}

We summarize our notations in the following Table.

\begin{table}[H]
\caption{\label{tab:pkl:Definitions-and-Notations}Definitions and Notations}

\centering \begin{tabular}{|c|c|c|c|}
\hline 
Notation & Description & Definition & Dimensions\tabularnewline
\hline 
\hline 
$n_{x}$ & state dimension &  & \tabularnewline
\hline 
$n_{y}$ & measurement dimension &  & \tabularnewline
\hline 
$A_{k}$ & system dynamics &  & $n_{x}\times n_{x}$\tabularnewline
\hline 
$C_{k}$ & measurement function &  & $n_{y}\times n_{x}$\tabularnewline
\hline 
$Q_{k},R_{k}$ & noise covariances &  & $n_{x}\times n_{x}$,$n_{y}\times n_{y}$\tabularnewline
\hline 
$x_{k}$ & system state (ground-truth) &  & $n_{x}$\tabularnewline
\hline 
$y_{k}$ & measurements &  & $n_{y}$\tabularnewline
\hline 
$\hat{x}_{k}$ & state estimator &  & $n_{x}$\tabularnewline
\hline 
$\hat{x}_{k}^{*}$ & optimal Kalman state & see Algorithm \ref{alg:Kalman-Filter} & $n_{x}$\tabularnewline
\hline 
$\hat{x}_{k|s}$ & best MSE state esimators, up to time $s$ &  & $n_{x}$\tabularnewline
\hline 
$\II_{k}$ & innovation process & see Algorithm \ref{alg:Kalman-Filter} & $n_{y}$\tabularnewline
\hline 
$S_{k}$ & innovation covariance & see Algorithm \ref{alg:Kalman-Filter} & $n_{y}\times n_{y}$\tabularnewline
\hline 
$K_{k}$ & Kalman gain & see Algorithm \ref{alg:Kalman-Filter} & $n_{x}\times n_{y}$\tabularnewline
\hline 
$\Pi_{k}$ & innovation perceptual gain &  & $n_{x}\times n_{x}$\tabularnewline
\hline 
$M_{k}$ & Kalman update covariance & $M_{k}=K_{k}S_{k}K_{k}^{\T}$ & $n_{x}\times n_{x}$\tabularnewline
\hline 
$\UII_{k}$ & unutilized information process & see     \eqref{eq::pkl:uii_k}
 & $kn_{y}$\tabularnewline
\hline
$\varUpsilon_{k}$ & unutilized information process (recursive) & see \eqref{eq::pkl:ups_k} & $n_{x}$\tabularnewline
\hline

$\Sigma_{\varUpsilon_{k}}$ & unutilized information covariance &  & $n_{x}\times n_x $\tabularnewline
\hline 
$\varPhi_{k}$ & unutilized information perceptual gain &  & $n_{x}\times n_{x}$\tabularnewline
\hline
$B_{k}$ & weight matrix & $B_{k}=\sum_{t=k}^{T}\w_{t}(A^{t-k})^{\T}A^{t-k}$ & $n_{x}\times n_{x}$\tabularnewline
\hline 
$D_{k}$ & deviation from MMSE & $D_{k}=\mathbb{E}\left[\hat{x}_{k}^{*}-\hat{x}_{k}\right]\left[\hat{x}_{k}^{*}-\hat{x}_{k}\right]^{\T}$ & $n_{x}\times n_{x}$\tabularnewline
\hline 
$T$ & Termination time (horizon) &  & \tabularnewline
\hline 
$\C(\hat{X}_0^T)$ & minimization objective &  $\C = 
\sum_{k=0}^{T}\w_{k}\EEb{\|{x}_{k}-\hat{x}_{k}\|^{2}} $ &\tabularnewline
\hline 
\end{tabular}

\end{table}

%% file: appendix/pk_demo.tex
\section{Numerical demonstrations}\label{app:numerical}

In this section we provide full details for the experimental settings of Sec.~\ref{sec::numericals}, with additional numerical and visual results. 
In the following, we compare the performance of several filters; $\hat{x}^*_{\kal}$ and $\hat{x}_{\ntc}$  correspond to the Kalman filter and the temporally-inconsistent filter \eqref{eq::pkl:ntc-estimator} (which does not possess perfect-perceptual quality). The estimate 
 $\hat{x}_{\opt}$ is generated by a  perfect-perception filter obtained by numerically optimizing the coefficients in \eqref{eq::pkl:J_k::fullJk}, where the cost is the MSE at termination time, \textit{i.e.}~the \textit{terminal cost}
\begin{equation}
\C_{\mathrm{T}}=\EEb{\|\hat{x}_{T}-x_{T}\|^{2}}.\label{appeq::pkl:TerminalCost}
\end{equation}
The estimates $\hat{x}_{\auc},\hat{x}_{\minT}$ correspond to PKF outputs (Alg.~\ref{alg:Perceptual-Kalman-Filter}) minimizing the \textit{total cost} (area under curve) 
\begin{equation}
\C_{\auc} = \sum_{k=0}^{T}\EEb{\|\hat{x}_{k}-x_{k}\|^{2}},
\label{appeq::pkl:TotalCost}
\end{equation} 
and the \textit{terminal cost}, respectively.
Finally, $\hat{x}_{\stat}$ is the stationary PKF, discussed in App.~\ref{app::stationary}.
The filters are  summarized in Table \ref{apptab::pkl:filters demonstrated}. 

\begin{table}[H]
\caption{List of demonstrated filters.}
    \label{apptab::pkl:filters demonstrated}
    \centering
    \begin{tabular}{c|c|c|c|c|c}
         & \multirow{2}{*}{\textbf{description}} & \multirow{2}{*}{\textbf{definition}}& \multicolumn{2}{c|}{\textbf{perfect-perception}} 
         \\ 
         & & & per-sample & temporal 
         \\
         \hline \rowcolor{Gray}
         $\hat{x}_{\kal}^*$& Kalman filter & Algorithm \ref{alg:Kalman-Filter} & \xmark & \xmark 
         \\
         \multirow{2}{*}{$\hat{x}_{\ntc}$} & Per-sample perceptual quality &\multirow{2}{*}{\eqref{eq::pkl:ntc-estimator} } & \multirow{2}{*}{ \cmark }& \multirow{2}{*}{\xmark} 
         \\
         & (temporally-inconsistent) & & & 
         \\ \rowcolor{Gray}
          & Optimized perfect-perceptual  & \multirow{2}{*}{\eqref{eq::pkl:J_k::fullJk} }& \multirow{2}{*}{\cmark} & \multirow{2}{*}{\cmark} 
         \\ \rowcolor{Gray}
         \multirow{-2}{*}{$\hat{x}_{\opt}$} &quality filter & \multirow{-2}{*}{\eqref{eq::pkl:J_k::fullJk} }& \multirow{-2}{*}{\cmark} & \multirow{-2}{*}{\cmark} 
         \\
         $\hat{x}_{\auc}$ & PKF with total cost & Algorithm \ref{alg:Perceptual-Kalman-Filter}& \cmark & \cmark 
         \\ \rowcolor{Gray}
         $\hat{x}_{\minT}$ & PKF with terminal cost & Algorithm \ref{alg:Perceptual-Kalman-Filter} & \cmark & \cmark 
         \\
         $\hat{x}_{\stat}$ & Stationary PKF & \eqref{eq::pkl:stationaryfilter}& \xmark & \xmark 
         \end{tabular}
    
\end{table}

\subsection{Example: Harmonic oscillator}

We start with a simple $2$-D example, where we demonstrate the  differences in MSE distortion between the optimized perfect-perceptual quality filter $\hat{x}_{\opt}$, the temporally inconsistent filter $\hat{x}_{\ntc}$ and the efficient sub-optimal (perceptual) PKF. 
Consider the harmonic oscillator, where the entries of the state $x_k \in \R^2$ correspond to position and velocity, and evolve as
\begin{equation}
x_{k+1}=Ax_{k}+q_{k},\quad q_{k}\sim\mathcal{N}\left(0, I\right)
\end{equation}
with
\begin{equation} 
A=I+\begin{bmatrix}   0 & 1 \\ -2 & 0 \end{bmatrix} 
\times \Delta_t,
\end{equation}
where $\Delta_t = 5 \times  {10}^{-3}$ is the sampling interval. Assume we have access to noisy and delayed scalar observations of the position (corresponding to time $t-\tfrac{1}{2}\Delta_t$) so that  $y_k = \begin{bmatrix} 1 & -\frac{1}{2} \end{bmatrix} x_k +r_k$, where $ r_k \sim \mathcal{N}\left(0,1\right) $ and $x_0 \sim \mathcal{N}\left(0,0.8 I\right) $.

We numerically optimize the coefficients $\left\{\Pi_k , \varPhi_k \right\}_{k=0}^T$ in \eqref{eq::pkl:J_k::fullJk}, to minimize the terminal error \eqref{appeq::pkl:TerminalCost} ($\tr{D_T}$ in \eqref{eq:D_k::Lyapunov_fullJk}) at time $T=255$ under the constraints \eqref{eq::pkl:Sigma_w_full}.
Figure \ref{appfig::harmonic_osc:MSE_fullview} shows the MSE distortion for the optimized perfect-perception filter $\hat{x}_{\opt}$ defined by \eqref{eq::pkl:J_k::fullJk} and $\left\{\Pi_k , \varPhi_k \right\}_{k=0}^T$, and the sub-optimal PKF outputs $\hat{x}_{\auc},\hat{x}_{\minT}$, minimizing the {total cost} 
\eqref{appeq::pkl:TotalCost}
and the {terminal cost}
\eqref{appeq::pkl:TerminalCost}
 (see Table \ref{apptab::pkl:filters demonstrated}). We observe that PKF estimations are indeed not MSE optimal at time $T$, However, their RMSE at time $T$ is only $\sim 30\%$ higher than that of $\hat{x}_{\opt}$ and they have the advantage that they can be solved analytically and require computing only half of the coefficients ($\Pi_k$).
 
The estimates $\hat{x}^*_{\kal}$ and $\hat{x}_{\ntc}$ achieve lower MSE than $\hat{x}_{\opt}$, however they do not possess perfect-perceptual quality. We can see the difference in MSE distortion between the filters $\hat{x}_{\opt}$ and $\hat{x}_{\ntc}$, with and without perception constraint in the temporal domain. \textit{This is the cost of temporal consistency in online estimation for this setting}.

\begin{figure}
\includegraphics[viewport=5bp 10bp 1375bp 675bp,clip,width=1.0\linewidth]
{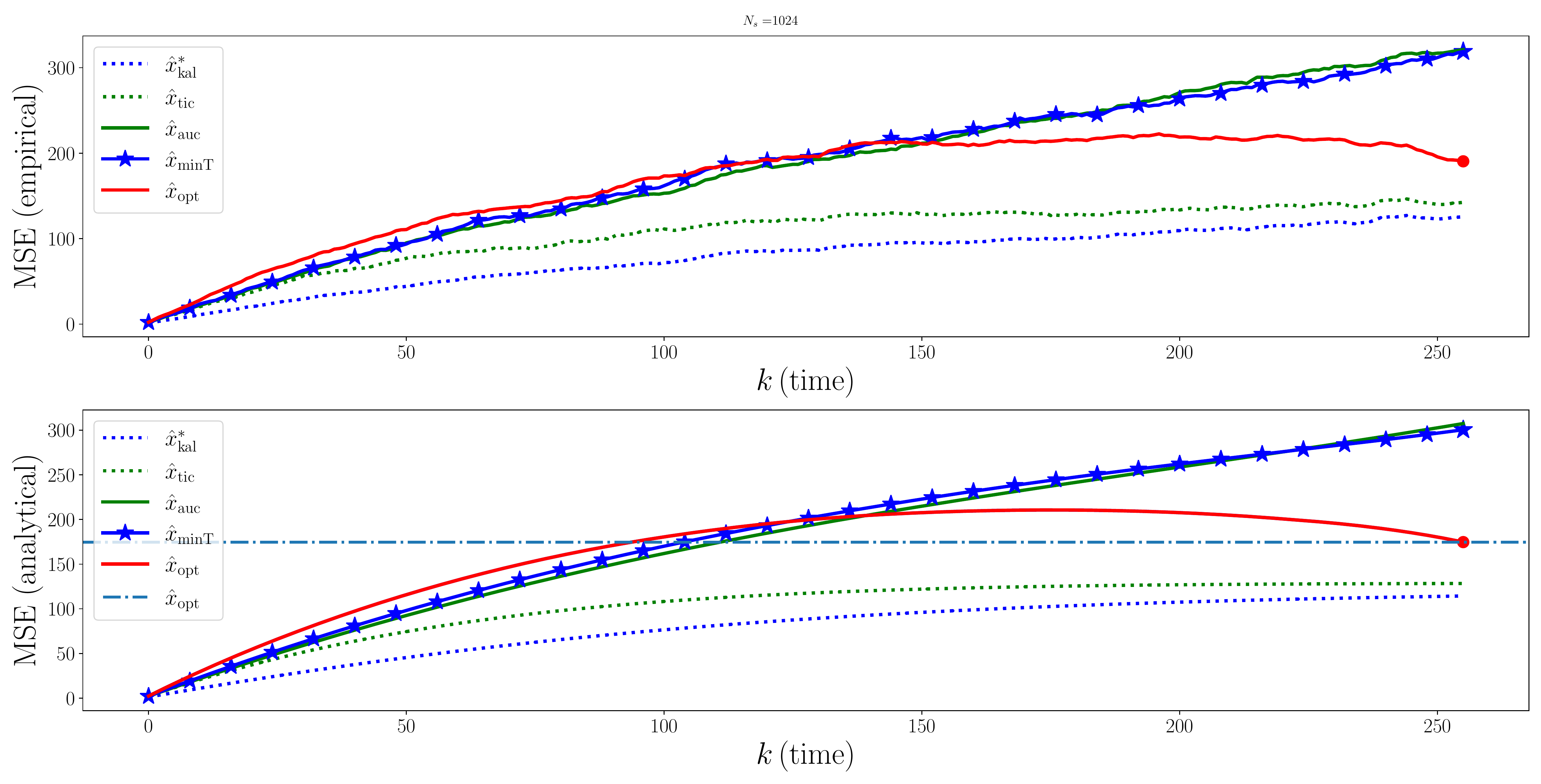}
\includegraphics[viewport=12bp 10bp 1380bp 675bp,clip,width=1.0\linewidth]
{plots/hosc/hosc603_RMSEgap.pdf}
\caption[MSE distortion on Harmonic oscillator for optimized perfect-perception filter vs. sub-optimal (PKF) estimations.]{
\textbf{MSE distortion on Harmonic oscillator.} \label{appfig::harmonic_osc:MSE_fullview}  $\hat{x}_{\opt}$ is a numerically optimized perfect-perceptual quality filter's output (minimizing error at time $T=255$, dashed horizontal line).  $\hat{x}_{\auc},\hat{x}_{\minT}$ are PKF outputs minimizing different objectives. Observe that PKF estimations are not MSE optimal, but require less computations. $\hat{x}^*_{\kal}$ and $\hat{x}_{\ntc}$  are not perfect-perceptual quality filters. 
\textbf{(top)} Empirical error, over $N=1024$ sampled trajectories.
\textbf{(bottom)} Analytical error. The difference in distortion between the perfect-perceptual state $\hat{x}_{\opt}$, optimized according to \eqref{eq::pkl:J_k::fullJk}, and $\hat{x}_{\ntc}$ is due to the perceptual constraint on the joint distribution. This is the cost of temporal consistency in online estimation for this setting. The gap between the MSE of the optimized estimator and $\hat{x}_{\minT}$ is due to the sub-optimal choice of coefficients, $\varPhi_k=0$.}
\end{figure}

\subsection{Example: Two coupled inverted pendulums}

Next, we demonstrate the quantitative behavior of perceptual Kalman filters, by comparing the MSE
between the PKF outputs when minimizing different cost functions, and between non-perceptual filters outputs. More specifically, this experiment demonstrates:
\begin{enumerate}
    \item How minimizing different objectives in Algorithm \ref{alg:Perceptual-Kalman-Filter} leads to different filters.
    \item The cost of perfect-perceptual quality filtering, given by Algorithm~\ref{alg:Perceptual-Kalman-Filter}, over optimal filters.
\end{enumerate}

We consider a higher-dimensional,
well-studied example of two coupled inverted pendulums, mounted on
carts \citep{freirich2018decentralized, freirich2016decentralized}. The cart positions, pendulum deviations, and
their velocities (Fig. \ref{fig::invPendulum}), are given by the discretized stable closed-loop system
with perturbation
\begin{equation}
x_{k+1}=Ax_{k}+q_{k},\quad q_{k}\sim\mathcal{N}\left(0,Q\right),
\end{equation}
where $x_k\in \mathbb{R}^8$. The initial state is distributed as
\begin{equation}
x_{0}\sim\mathcal{N}\left(0,P_{0}\right).
\end{equation}
The system matrices are given by 
\begin{equation}
A=I+A_{cl}\cdot\Delta_{t},
\end{equation}
 where $\Delta_{t}=5\times10^{-4}$ is the sampling interval and 
\begin{equation}
A_{cl}=\left[\begin{array}{cc}
A_{1}+BK_{1} & F\\
F & A_{2}+BK_{2}
\end{array}\right],
\end{equation}

\begin{equation}
A_{1}=A_{2}=\begin{bmatrix}0 & 1 & 0 & 0\\
2.9156 & 0 & -0.0005 & 0\\
0 & 0 & 0 & 1\\
-1.6663 & 0 & 0.0002 & 0
\end{bmatrix},\quad B=\begin{bmatrix}0\\
-0.0042\\
0\\
0.0167
\end{bmatrix}.
\end{equation}
The $\emph{{coupling}}$ is given by
\begin{equation}
F=\begin{bmatrix}0 & 0 & 0 & 0\\
0.0011 & 0 & 0,0.0005 & 0\\
0 & 0 & 0 & 0\\
-0.0003 & 0 & -0.0002 & 0
\end{bmatrix},
\end{equation}
and stabilizing state-feedback controllers (each acts on a single cart) are
\begin{equation}
K_{1}=\begin{bmatrix}11396.0 & 7196.2 & 573.96 & 1199.0\end{bmatrix},\quad K_{2}=\begin{bmatrix}29241 & 18135 & 2875.3 & 3693.9\end{bmatrix}.
\end{equation}

The partial measurements are given by $y_{k}=C x_{k}+r_{k}$, where  $r_{k}\sim\mathcal{N}\left(0,R\right)$,
with coefficients
\begin{equation}
C=\left[\begin{array}{cc}
\bar C_{1} & 0\\
0 & \bar C_{2}
\end{array}\right],\quad \bar C_{1}=\bar C_{2}=\begin{bmatrix}1 & 0 & 0 & 0\\
0 & 0 & 1 & 0
\end{bmatrix}.
\end{equation}
Namely, we observe only position and angle for each cart/pendulum,
while velocities are not being measured.

The perturbation covariances are given by 
\begin{equation}
P_{0}=\left[\begin{array}{cc}
\bar{P}_{0} & 0\\
0 & \bar{P}_{0}
\end{array}\right],\quad Q=\left[\begin{array}{cc}
\bar{Q} & 0\\
0 & \bar{Q}
\end{array}\right],\quad R=\left[\begin{array}{cc}
\bar{R} & \frac{1}{8}\bar{R}\\
\frac{1}{8}\bar{R} & \bar{R}
\end{array}\right],
\end{equation}
where
\begin{equation}
\bar{P}_{0}=\begin{bmatrix}0.154 & 0.142 & -0.143 & 0.093 \\ 0.142 & 0.144 & -0.124 & 0.058 \\ -0.143 & -0.124 & 0.167 & -0.148 \\ 0.093 & 0.058 & -0.148 & 0.192 \end{bmatrix}\cdot 5 \times10^{-4},
\end{equation}
\begin{equation}
\bar{Q}={10}^{-2}\cdot\begin{bmatrix}0.642 & -0.136 & 0.78 & 0.262 \\ -0.136 & 0.894 & -0.248 & 0.074 \\ 0.78 & -0.248 & 1.284 & -0.314 \\ 0.262 & 0.074 & -0.314 & 1.766\end{bmatrix}\times \Delta_t,\,
\bar{R}={10}^{-2}\cdot\begin{bmatrix}0.375 & -0.33 \\ -0.33 & 0.771\end{bmatrix}\times \Delta_t.
\end{equation}

\begin{figure}[H]
\centering \includegraphics[width=0.75\linewidth]{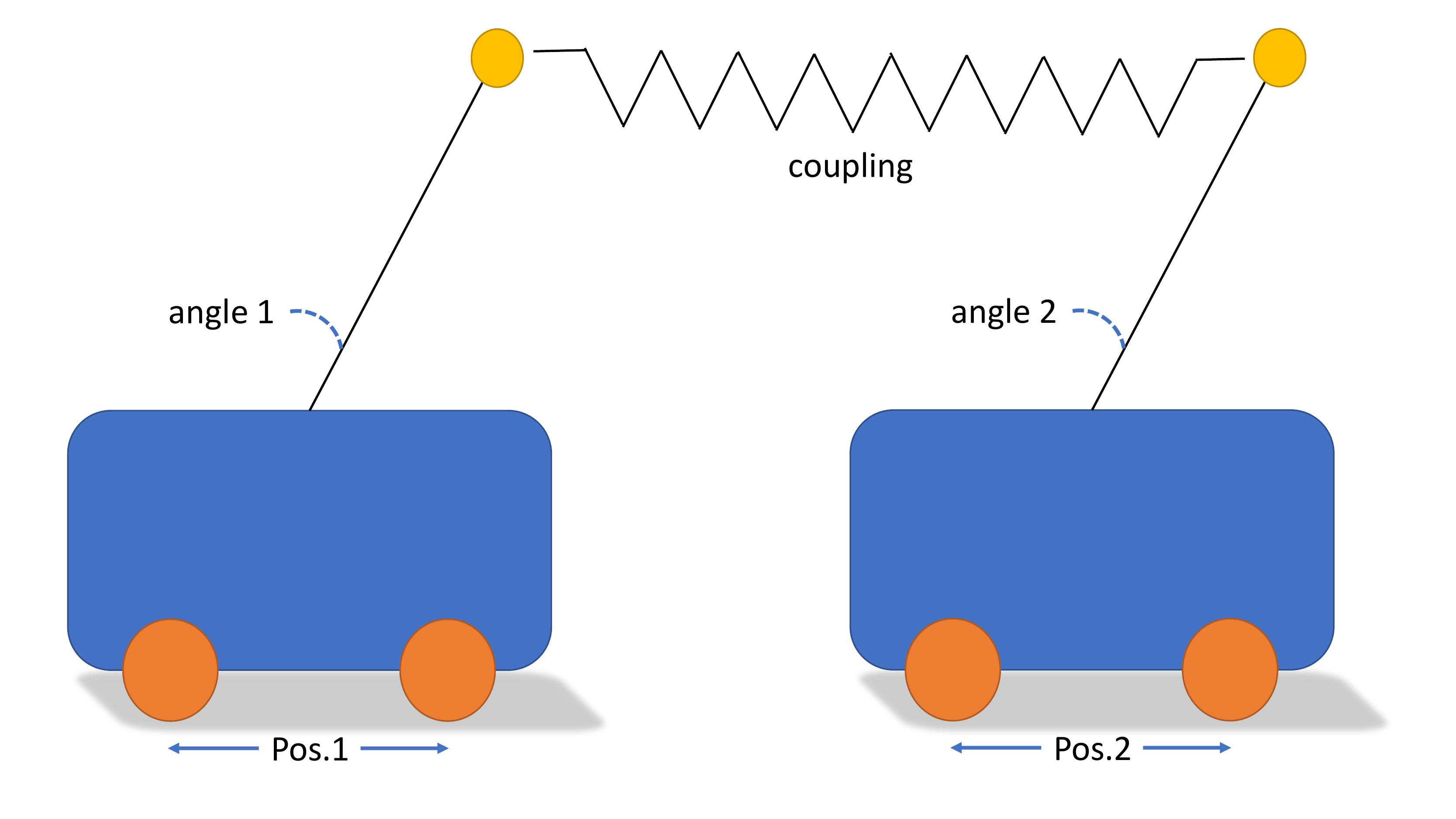}\caption[Coupled inverted pendulums.]{\textbf{Coupled inverted pendulums.}}\label{fig::invPendulum}
\end{figure}

We simulate the system for $2^{10}$ time steps ($T=2^{10}-1$), over $N=2^{10}$ independent experiments.
In Figures~\ref{appfig::invPendulum:MSE} and~\ref{appfig::invPendulum:MSE_fullview} we show the MSE distortion as a function of time, $\EEb{\|\hat{x}_k -x_k\|^2}$, for the different
filters of Table \ref{apptab::pkl:filters demonstrated}; $\hat{x}_{\kal}^{*}$ is the optimal Kalman filter. $\hat{x}_{\ntc}$
is the perceptual filter without consistency constraints, given in \eqref{eq::pkl:ntc-estimator}. $\hat{x}_{\auc}$ is the PKF output minimizing the {total cost} \eqref{appeq::pkl:TotalCost}.
$\hat{x}_{\minT}$ (marked by `$\star$') is the PKF output minimizing the terminal cost \eqref{appeq::pkl:TerminalCost}.

\begin{figure}
\includegraphics[viewport=5bp 5bp 1165bp 685bp,clip,width=0.99\linewidth]
{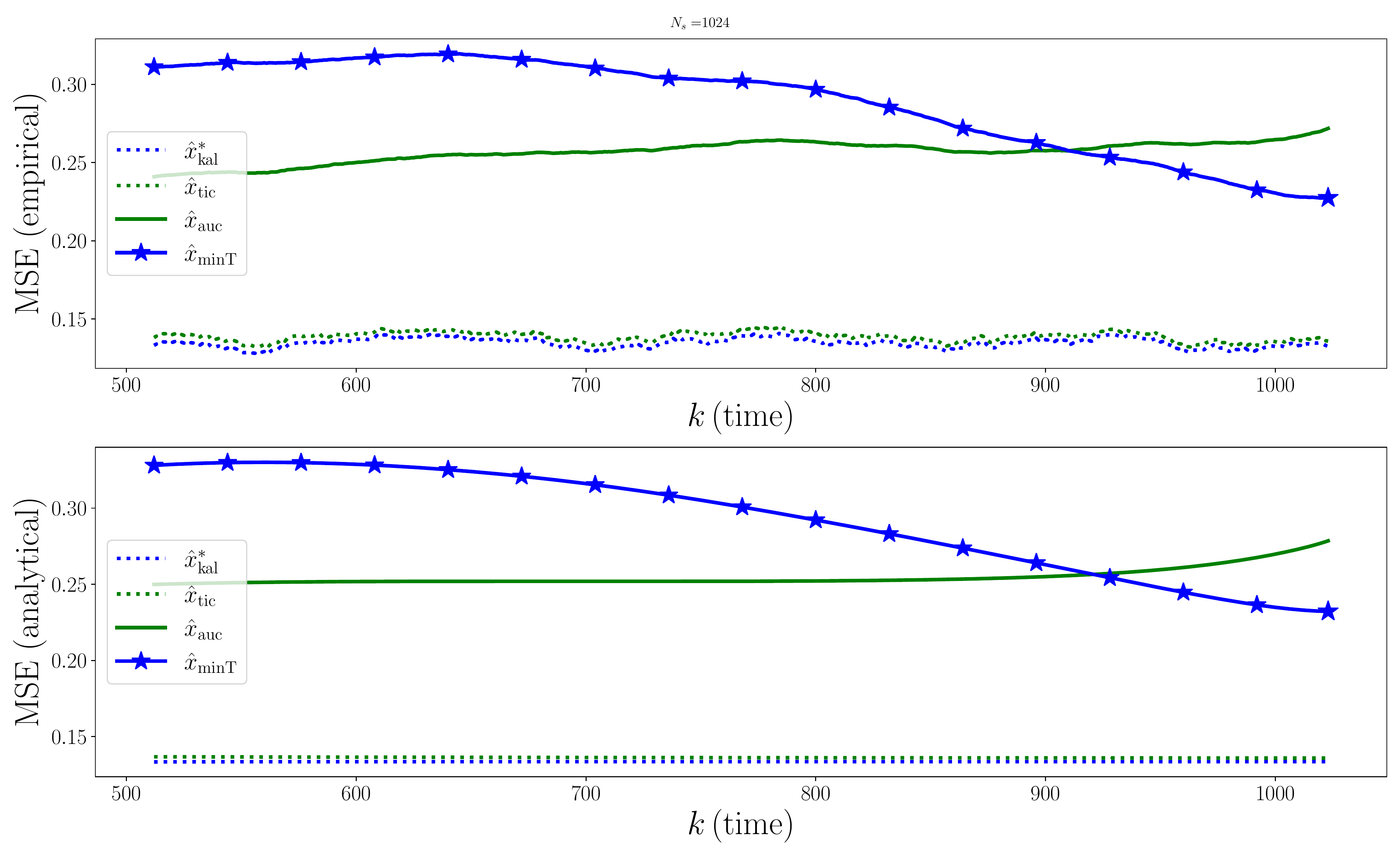}
\caption[MSE distortion on coupled inverted pendulums for perceptual and non-perceptual filters (near the time $T$).]{\textbf{MSE distortion on Coupled inverted pendulums for perceptual and non-perceptual filters (near the time $T$).} \label{appfig::invPendulum:MSE}
$\hat{x}_{\auc},\hat{x}_{\minT}$ are PKF outputs minimizing different objectives. Observe that while both possess perfect-perceptual quality, they yield different estimations. Also, pay attention to the MSE gap between the MSE-optimal, but not perceptual, Kalman filter and the PKFs.}
\end{figure}

\begin{figure}
\includegraphics[viewport=5bp 5bp 1165bp 685bp,clip,width=0.99\linewidth]
{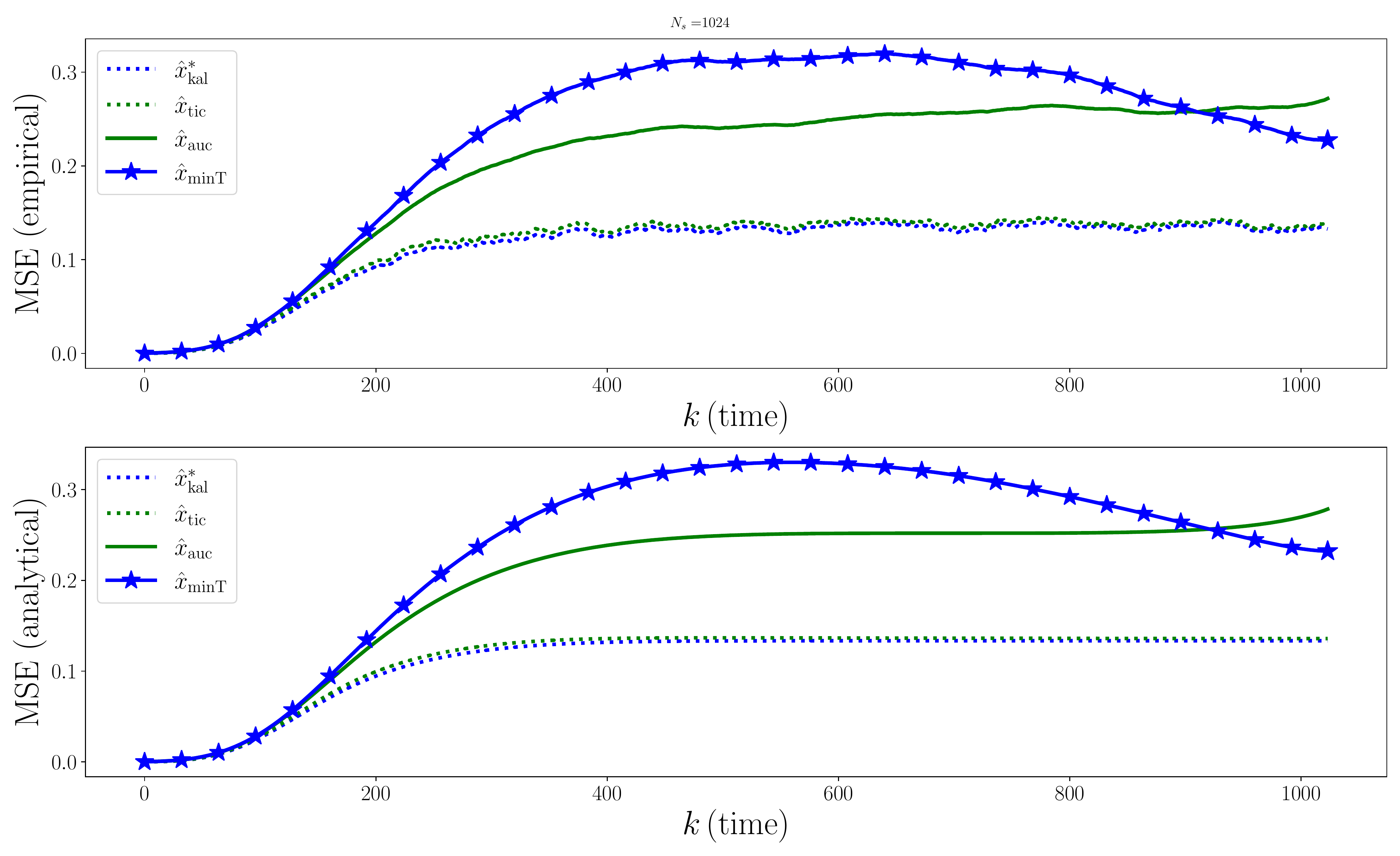}
\caption[MSE distortion on coupled inverted pendulums for perceptual and non-perceptual filters (full view).]{
\textbf{MSE distortion on Coupled inverted pendulums for perceptual and non-perceptual filters (full view).}} \label{appfig::invPendulum:MSE_fullview}
\end{figure}

We observe that filters satisfying the perfect perceptual quality
constraint ($\hat{x}_{\auc}$ and $\hat{x}_{\minT}$) achieve higher distortions
compared to the per-sample only perceptual filter $\hat{x}_{\ntc}$, which in turn attains MSE distortion slightly higher than that of the MSE-optimal Kalman filter. This
demonstrates again the cost of temporal consistency in online estimation. Note also
that PKFs minimizing different cumulative objectives, yield different
estimations; while $\hat{x}_{minT}$ is optimal at termination time $T$, $\hat{x}_{\auc}$ achieves a lower MSE on average. As we will see next, both filters attain the same perceptual quality.

\begin{figure}
\includegraphics[viewport=5bp 5bp 1165bp 668bp,clip,width=0.95\linewidth]{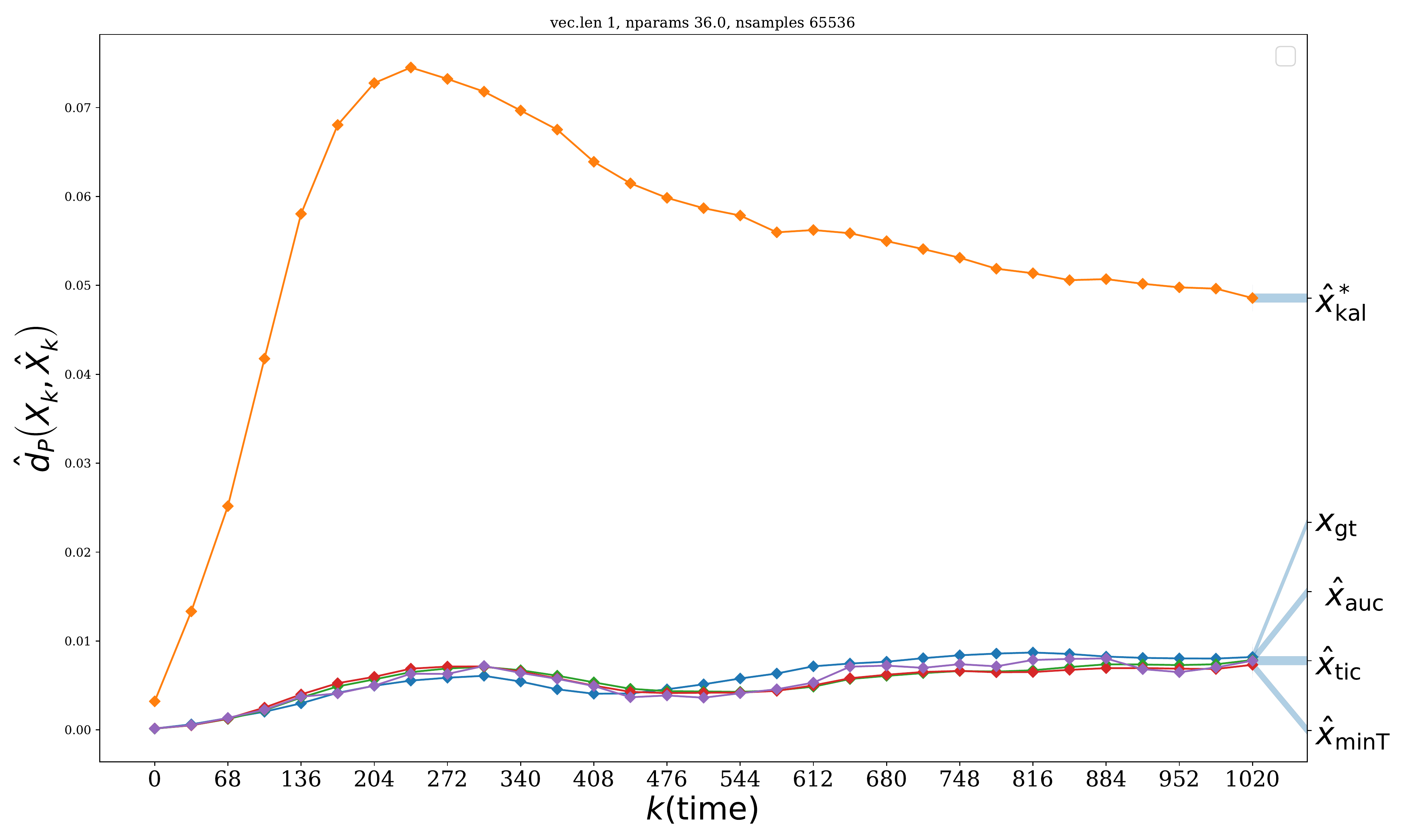}
\includegraphics[viewport=5bp 5bp 1165bp 668bp,clip,width=0.95\linewidth]{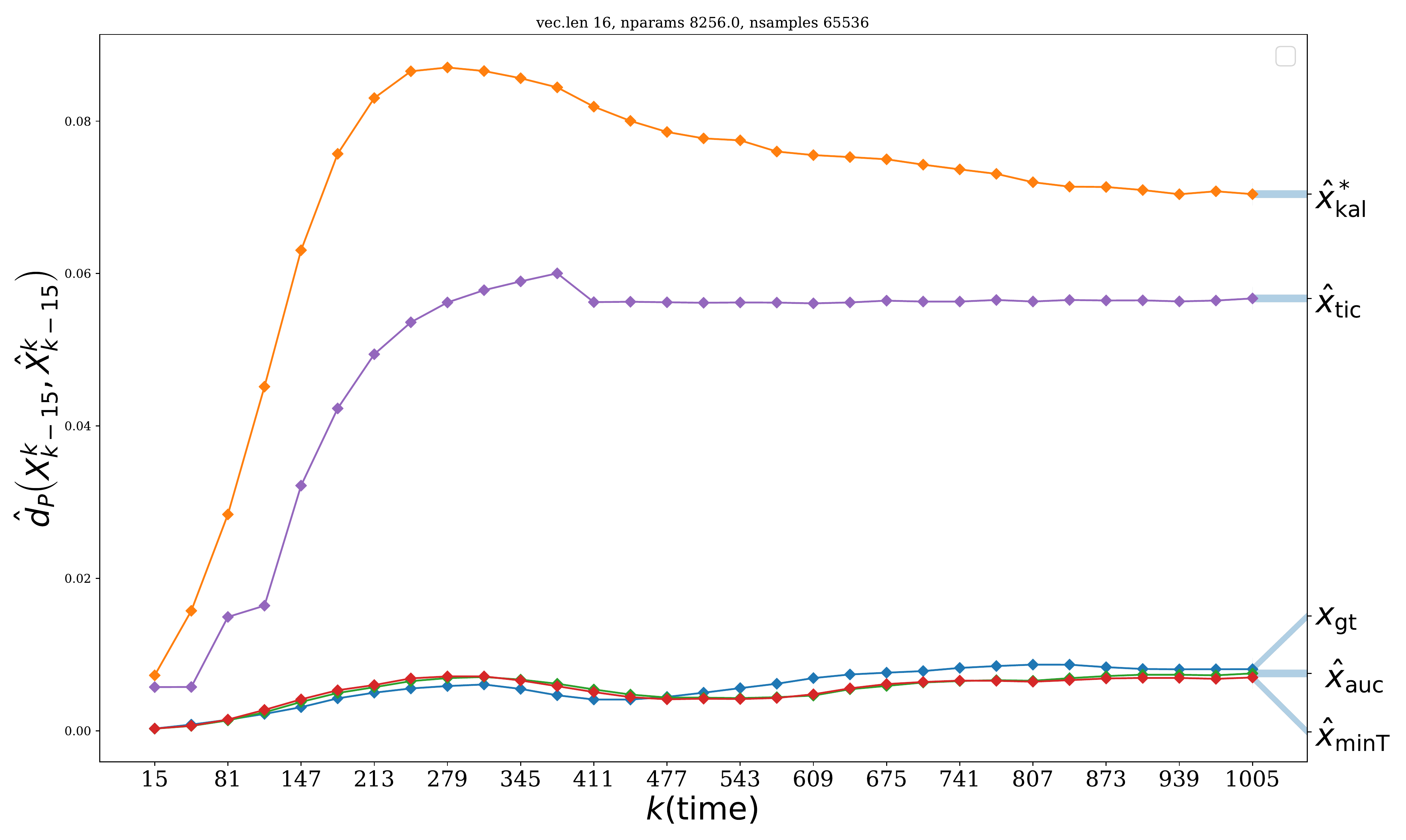} \caption[Perceptual quality measured by empirical Gelbrich distance.]{ \textbf{Perceptual quality measured by estimated Wasserstein distance $\hat{d}_P$ (lower is better).} \label{appfig::invPendulum:Gelbrich}
{\textbf{(top)}} Distance between distributions of single samples $p_{x_k}$ and $p_{\hat{x}_k}$. {\textbf{(bottom)}} Distance between distributions of $16$-state vectors (at times $[k-15,k]$), $p_{{X}^k_{k-15}}$ and $p_{\hat{X}^k_{k-15}}$. Observe that $\hat{x}_{\ntc}$ single samples are distributed similarly to the ground-truth signal, but they fail to attain the reference joint distribution between timesteps. PKF outputs $\hat{x}_{\auc}$ and $\hat{x}_{\minT}$ attain high measured quality in both cases. }
\end{figure}

In Fig.~\ref{appfig::invPendulum:Gelbrich}  we estimate the perceptual quality, given by the Wasserstein
distance 
between the ground-truth distribution and the empirical Gaussian
distributions of the different filters outputs. In Fig. \ref{appfig::invPendulum:Gelbrich}(top) we
estimate the distance between single-sample distributions, while in
Fig. \ref{appfig::invPendulum:Gelbrich}(bottom) we consider the joint distributions of $16$ state-vectors,
$x_{t},t\in[k-15,k]$. Observe that while each sample of $\hat{x}_{\ntc}$
is distributed similarly to its reference sample, it fails to attain
perfect perceptual quality where we measure the distance from the real
process distribution. PKF outputs attain low perceptual index (high quality) in both scenarios. We also present the perceptual quality measured
for the ground-truth signal $x_{\gt}$ empirical distribution, as a reference. 

Figure \ref{appfig::invPendulum:MSEstat} shows the asymptotic behavior (empirical error for large horizon $T$) of $\hat{x}_{stat.}$, the stationary PKL \eqref{eq::pkl:stationaryfilter}. The figure also presents the empirical errors for Kalman filter and its stationary version (multiplied by a factor of $2$, which is an upper bound on the MSE distortion of perceptual estimators without temporal constraints, see \citep{blau2018perception}), and the theoretical steady-state error of \eqref{eq::pkl:stationaryfilter}, obtained by optimizing \eqref{eq::pkl:stationaryMaxOptimization} (dashed horizontal line) for comparison. The error of the non-stationary perceptual filter $\hat{x}_{\auc}$ is also shown.

\begin{figure}
\includegraphics[viewport=5bp 5bp 1165bp 670bp,clip,width=0.99\linewidth]{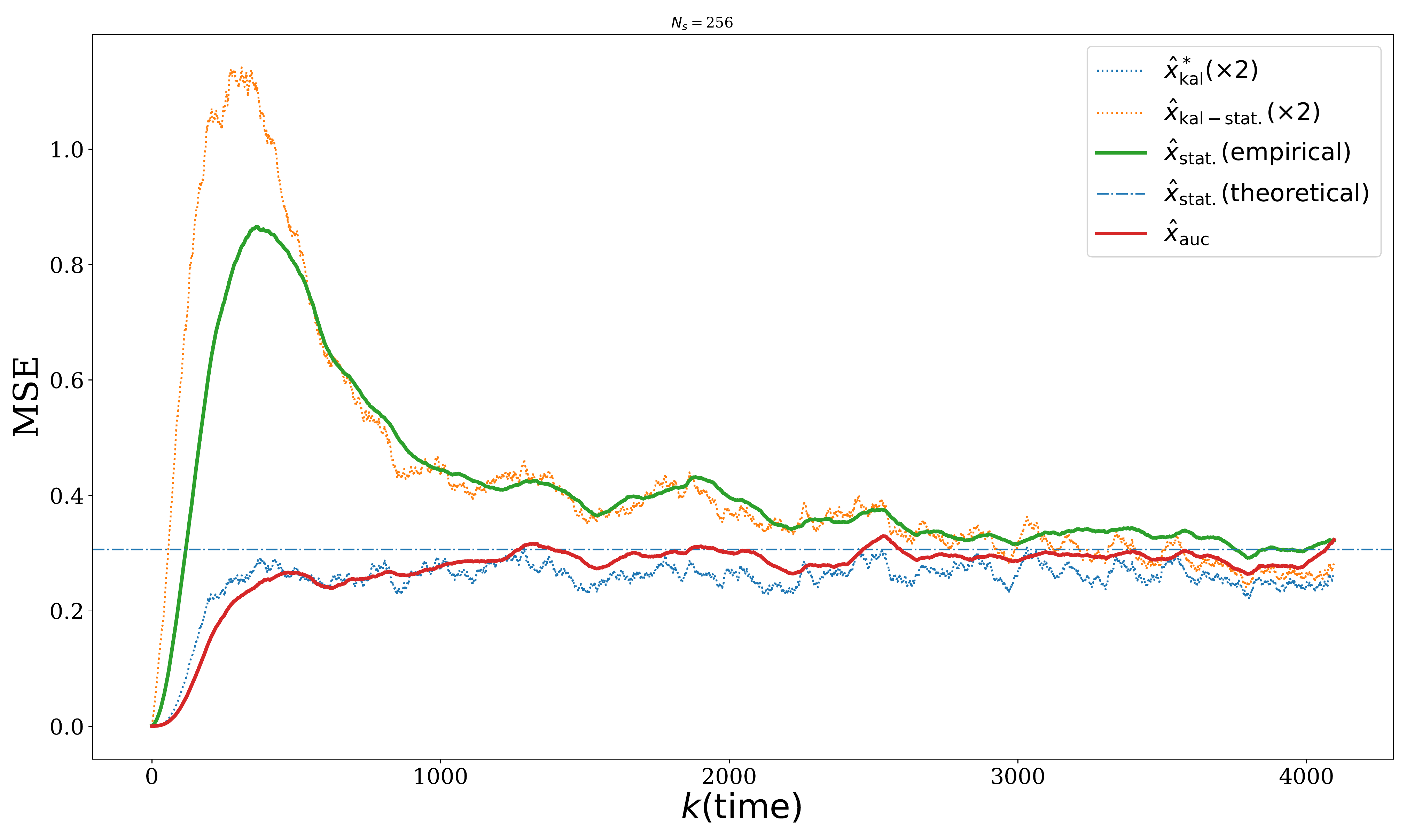} \caption[MSE distortion on Coupled inverted pendulums for stationary filters.]{
\textbf{MSE distortion on Coupled inverted pendulums for stationary filters.}} \label{appfig::invPendulum:MSEstat}
\end{figure}

\input{appendix/dyntex_full}

%% file: appendix/dyntex_full.tex
\subsection{Dynamic texture}

Here we illustrate the qualitative effects of perceptual (temporally consistent) estimation in a simplified video restoration setting. Please see the supplementary video for the full videos. 
This setup visually demonstrates how:
\begin{enumerate}
    \item Filters with no perfect perceptual quality tend to generate non-realistic images or atypical motion (random or slow movement, flickering artifacts etc.).  
    \item PKF outputs are natural to the domain, both spatially and temporally.
\end{enumerate}

For this extent, we introduce the `Dynamic Texture' domain. In this domain, video frames are generated from a latent state which represents their \textit{Factor-Analysis} (FA) decomposition (see \textit{e.g.} \citet[Sec. 12.2.4]{bishop2006pattern} for more details). The dynamics in the FA domain are assumed to be 
linear, with a small Gaussian perturbation, 
\begin{equation}
    x^{\FA}_k = A^{\FA} x^{\FA}_{k-1} + q_k,\quad x_0^{\FA}\sim\mathcal{N}\left(0,I\right), \quad x^{\FA}_k\in\R^{128}.
\end{equation}
The state vector with the given  dynamics creates frames of a wavy lake in the video domain \footnote{Original frames are taken from `river-14205' by OjasweinGuptaOJG via pixabay.com, and are free to use under the content licence.}, through an affine transformation, 
\begin{equation}
    x^{vid}_k = W_{\FA\rightarrow vid}\left(x_k^{\FA} + \varepsilon^{\FA}\right).
\end{equation}
$W_{\FA\rightarrow vid}$ is a linear transformation from $\R^{128}$ latent states to $\R^{512\times512\times3}$ frames, and $\varepsilon^{\FA}$ is a constant vector. 
$A^{FA}$ and the noise $q_k$ parameters are estimated similarly to \cite{doretto2003dynamic}.
Linear observations $y_k \in \R^{32\times 32}$ are given in the frame (pixel) domain, by 
\begin{equation}
y_k = C_{k} x^{\FA}_{k} + r_k.
\label{eq::pkl:dyntex:obs}
\end{equation}
At times where information is being observed,
\begin{equation}
\label{appeq::pkl:C:obs}
    C_k = C_{\times 16}W_{RGB\rightarrow y}W_{\FA\rightarrow vid},
\end{equation}
where $W_{RGB\rightarrow y}$ is a projection onto the $Y$-channel (grayscale) and $C_{\times 16}$ is a matrix that performs $16\times$ downsampling in both axes. At times where there is no observed information, $C_k=0$. Here, $r_k$ is a Gaussian noise.

In our first  experiment, measurements are supplied as in \eqref{appeq::pkl:C:obs}
 up to frame $k=127$ and then vanish ($C_k=0,k\geq128$), letting the different filters predict the next, unobserved, frames of the sequence. 
We pass $y_k$ as an input to the various filters (see Table \ref{apptab::pkl:filters demonstrated}); $\hat{x}^*_{\kal}$ is the Kalman filter output. $\hat{x}_{\ntc}$ is the perceptual filter in the spatial domain, given in \eqref{eq::pkl:ntc-estimator}. $\hat{x}_{\auc}$ is our Algorithm (PKF) output reducing the total cost in the latent space,
$
    \C_{\auc}=\sum_{k=0}^T \EEb{\| x^{\FA}_k - \hat{x}_k\|^2}
$. 
All filtering is done in the latent domain, and then transformed to the pixel domain. MSE is also calculated in the FA domain.
In (Fig. \ref{appfig::pkf:river_demo:predictexp}) we can see that until frame $k = 127$, all filters reconstruct the reference frames well. Starting at time $k=128$, when measurements disappear, we observe that the Kalman filter slowly fades into a static, blurry output which is the average frame value in this setting. This is definitely a non-`realistic' video; Neither the individual frames nor the static behavior are natural to the domain.
Our perfect-perceptual filter, $\hat{x}_{\auc}$, keeps generating a `natural' video, both spatially and temporally. This makes its MSE grow faster.

\begin{figure}
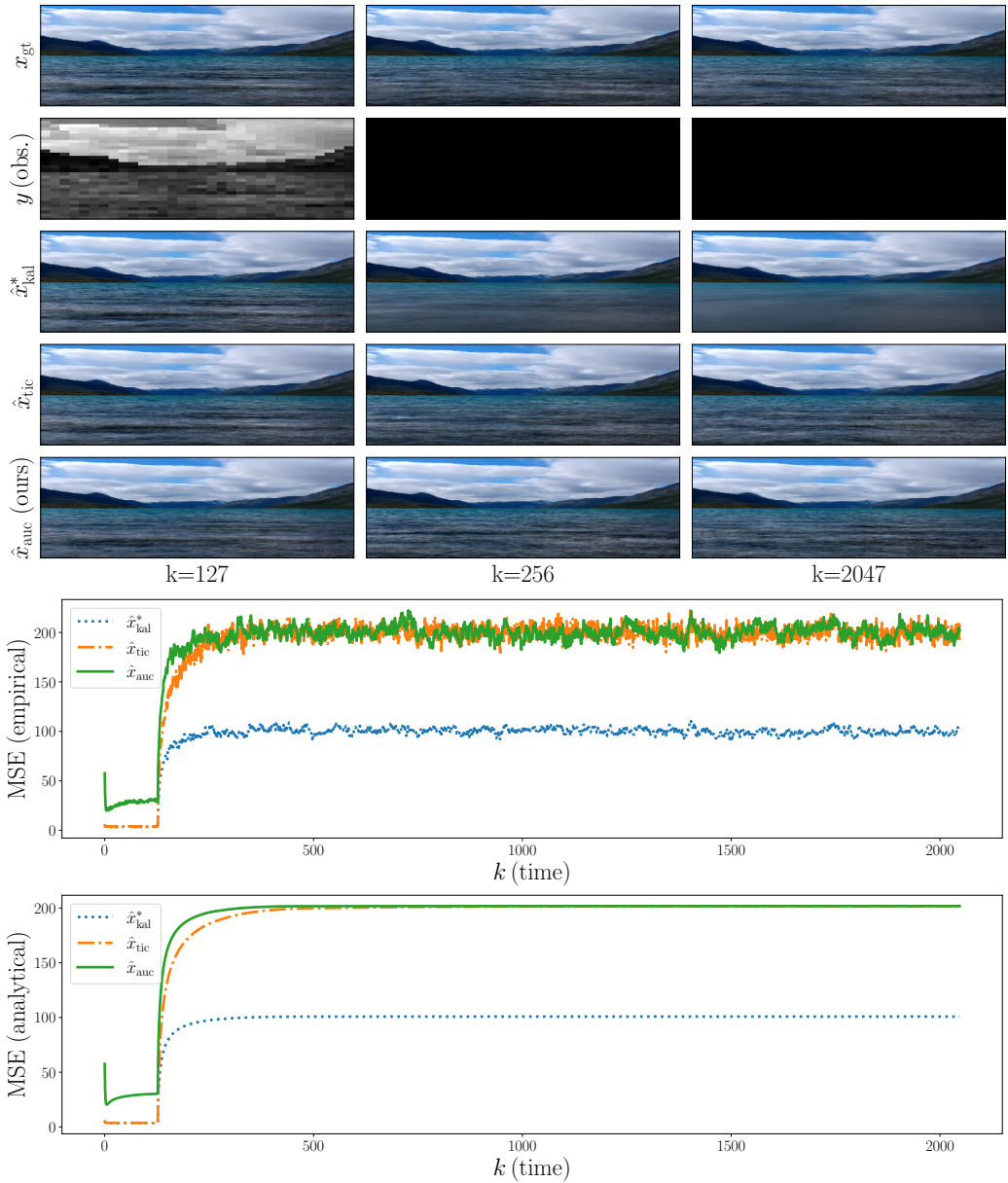

\centering \centering 
\includegraphics[viewport=5bp 5bp 1165bp 685bp,clip,width=0.99\linewidth]{plots/river/river_wc10_ours.pdf}
\includegraphics[viewport=5bp 5bp 1165bp 685bp,clip,width=0.99\linewidth]{plots/river/river_wc10_oursMSE.pdf}
\caption{\label{appfig::pkf:river_demo:predictexp}
\textbf{Frame prediction on a dynamic texture domain.}   
In this experiment, measurements are supplied only up to frame $k=127$. The filter's task here is to predict the unobserved future frames of the sequence.  Observe that the $\hat{x}^*_{\kal}$ fades into a  blurred average frame, while the perceptual filter $\hat{x}_{\auc}$ generates a natural video, both spatially and temporally. This makes its MSE grow faster,
}     
\end{figure}

We now perform a second experiment, where 
$C_k$ is
set to zero until frame $k=512$.  At times $k \geq 513$ 
measurements are
given again by the noisy, downsampled frames as described in \eqref{eq::pkl:dyntex:obs}-\eqref{appeq::pkl:C:obs}.
In Fig.~\ref{appfig::pkf:river_demo:genexp} we present the outcomes of the different filters. 
We first observe that up to frame $k=512$, there is no observed information, hence outputs are actually being generated according to priors.
The Kalman filter outputs  a static, average frame. 
$\hat{x}_{\ntc}$ randomizes each frame independently, which creates  the impression of rapid, random movement with flickering features, which is unnatural to the reference domain.
At frame $k=513$, when observations become available, we can see that $\hat{x}^*_{\kal}$ and $\hat{x}_{\ntc}$ are being updated immediately, creating an inconsistent, non-smooth motion between frames $512$ and $513$. PKF output $\hat{x}_{\auc}$, on the other hand, keeps  maintaining a smooth motion. 
Since non-consistent filters outputs rapidly becomes similar to the ground-truth, their errors drop. The perfect-perceptual filter, $\hat{x}_{\auc}$, remains consistent with its previously generated frames and the natural dynamics of the model, hence its error decays more slowly.

\begin{figure}
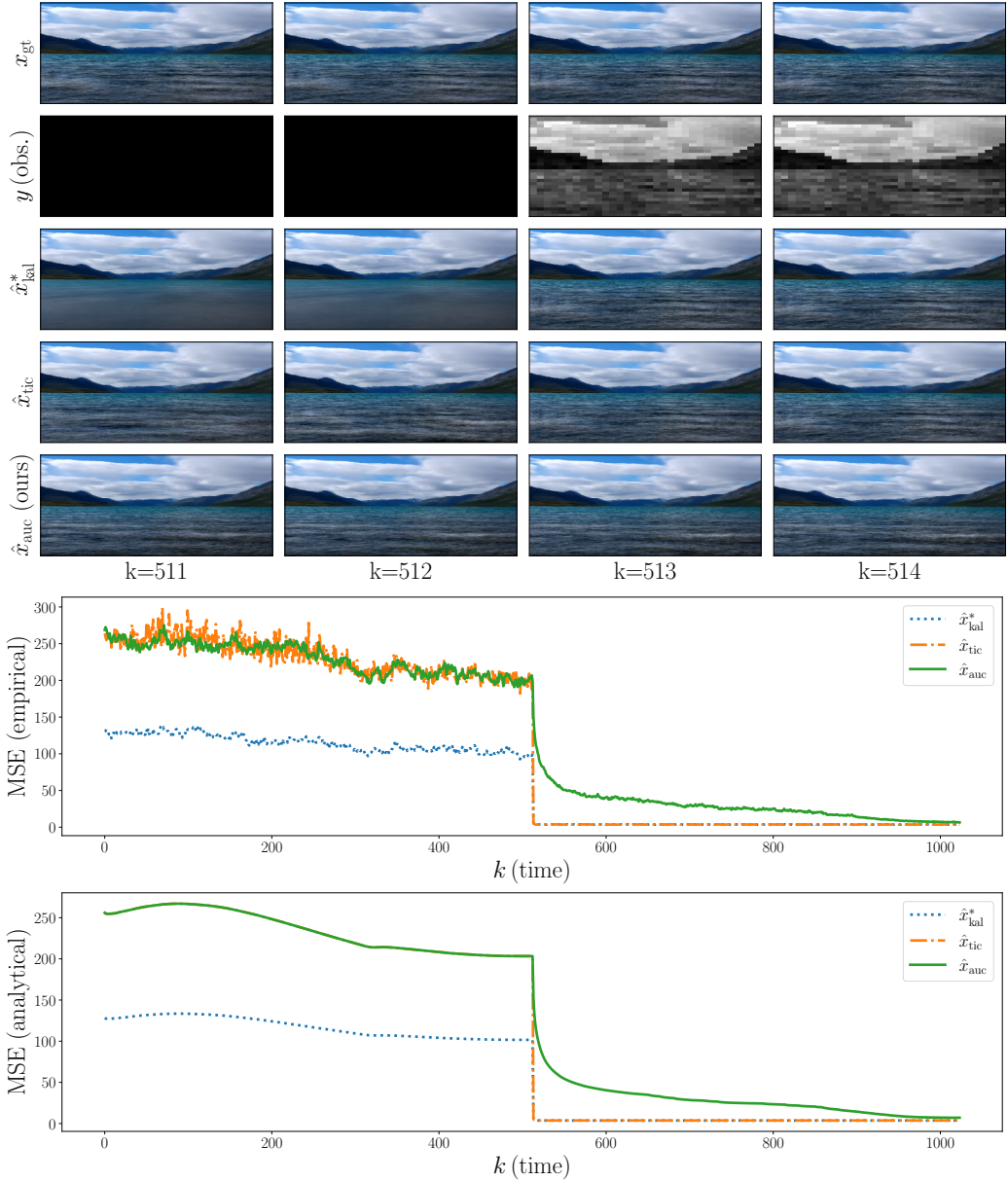

\centering \centering 
\includegraphics[viewport=5bp 5bp 1165bp 685bp,clip,width=0.99\linewidth]{plots/river/river_wc5_T1024.pdf}
\includegraphics[viewport=5bp 5bp 1165bp 685bp,clip,width=0.99\linewidth]{plots/river/river_wc5_T1024MSE.pdf}
\caption{\label{appfig::pkf:river_demo:genexp}
\textbf{Frame generation on Dynamic texture domain.}   
In the first half of the demo ($k\leq 512$), there are no observations, hence the reference signal is restored according to prior distribution. We observe that filters with no perfect-perceptual quality constraint in the temporal domain generate non-realistic frames (Kalman filter  output $\hat{x}_{\kal}^*$) or unnatural motion ($\hat{x}_{\ntc}$).
Perceptual filter $\hat{x}_{\auc}$ is constrained by previously generated frames and the natural dynamics of the domain, hence its MSE decays slower.
}     
\end{figure}

%% file: main.bbl
\begin{thebibliography}{17}
\providecommand{\natexlab}[1]{#1}
\providecommand{\url}[1]{\texttt{#1}}
\expandafter\ifx\csname urlstyle\endcsname\relax
  \providecommand{\doi}[1]{doi: #1}\else
  \providecommand{\doi}{doi: \begingroup \urlstyle{rm}\Url}\fi

\bibitem[Bhattacharjee and Das(2017)]{bhattacharjee2017temporal}
Prateep Bhattacharjee and Sukhendu Das.
\newblock Temporal coherency based criteria for predicting video frames using
  deep multi-stage generative adversarial networks.
\newblock \emph{advances in neural information processing systems}, 30, 2017.

\bibitem[Bishop and Nasrabadi(2006)]{bishop2006pattern}
Christopher~M Bishop and Nasser~M Nasrabadi.
\newblock \emph{Pattern recognition and machine learning}, volume~4.
\newblock Springer, 2006.

\bibitem[Blau and Michaeli(2018)]{blau2018perception}
Yochai Blau and Tomer Michaeli.
\newblock The perception-distortion tradeoff.
\newblock In \emph{Proceedings of the IEEE Conference on Computer Vision and
  Pattern Recognition}, pages 6228--6237, 2018.

\bibitem[Chu et~al.(2020)Chu, Xie, Mayer, Leal-Taix{\'e}, and
  Thuerey]{chu2020learning}
Mengyu Chu, You Xie, Jonas Mayer, Laura Leal-Taix{\'e}, and Nils Thuerey.
\newblock Learning temporal coherence via self-supervision for gan-based video
  generation.
\newblock \emph{ACM Transactions on Graphics (TOG)}, 39\penalty0 (4):\penalty0
  75--1, 2020.

\bibitem[Doretto et~al.(2003)Doretto, Chiuso, Wu, and
  Soatto]{doretto2003dynamic}
Gianfranco Doretto, Alessandro Chiuso, Ying~Nian Wu, and Stefano Soatto.
\newblock Dynamic textures.
\newblock \emph{International Journal of Computer Vision}, 51:\penalty0
  91--109, 2003.

\bibitem[Freirich and Fridman(2016)]{freirich2016decentralized}
Dror Freirich and Emilia Fridman.
\newblock Decentralized networked control of systems with local networks: A
  time-delay approach.
\newblock \emph{Automatica}, 69:\penalty0 201--209, 2016.

\bibitem[Freirich and Fridman(2018)]{freirich2018decentralized}
Dror Freirich and Emilia Fridman.
\newblock Decentralized networked control of discrete-time systems with local
  networks.
\newblock \emph{International Journal of Robust and Nonlinear Control},
  28\penalty0 (1):\penalty0 365--380, 2018.

\bibitem[Freirich et~al.(2021)Freirich, Michaeli, and Meir]{freirich2021theory}
Dror Freirich, Tomer Michaeli, and Ron Meir.
\newblock A theory of the distortion-perception tradeoff in wasserstein space.
\newblock \emph{Advances in Neural Information Processing Systems},
  34:\penalty0 25661--25672, 2021.

\bibitem[Freund(2004)]{freund2004introduction}
Robert~M Freund.
\newblock Introduction to semidefinite programming (sdp).
\newblock \emph{Massachusetts Institute of Technology}, pages 8--11, 2004.

\bibitem[Ho et~al.(2022)Ho, Salimans, Gritsenko, Chan, Norouzi, and
  Fleet]{ho2022video}
Jonathan Ho, Tim Salimans, Alexey Gritsenko, William Chan, Mohammad Norouzi,
  and David~J Fleet.
\newblock Video diffusion models.
\newblock \emph{arXiv preprint arXiv:2204.03458}, 2022.

\bibitem[Kalman(1960)]{kalman1960filtering}
R.E. Kalman.
\newblock A new approach to linear filtering and prediction problems.
\newblock \emph{Journal of Basic Engineering. 82 (1) 35-45}, 1960.

\bibitem[Kim et~al.(2018)Kim, Sajjadi, Hirsch, and Scholkopf]{Kim_2018_ECCV}
Tae~Hyun Kim, Mehdi S.~M. Sajjadi, Michael Hirsch, and Bernhard Scholkopf.
\newblock Spatio-temporal transformer network for video restoration.
\newblock In \emph{Proceedings of the European Conference on Computer Vision
  (ECCV)}, September 2018.

\bibitem[Olkin and Pukelsheim(1982)]{olkin1982distance}
Ingram Olkin and Friedrich Pukelsheim.
\newblock The distance between two random vectors with given dispersion
  matrices.
\newblock \emph{Linear Algebra and its Applications}, 48:\penalty0 257--263,
  1982.

\bibitem[P{\'e}rez-Pellitero et~al.(2018)P{\'e}rez-Pellitero, Sajjadi, Hirsch,
  and Sch{\"o}lkopf]{perez2018perceptual}
Eduardo P{\'e}rez-Pellitero, Mehdi~SM Sajjadi, Michael Hirsch, and Bernhard
  Sch{\"o}lkopf.
\newblock Perceptual video super resolution with enhanced temporal consistency.
\newblock \emph{arXiv preprint arXiv:1807.07930}, 2018.

\bibitem[Shapiro(1985)]{shapiro1985extremal}
Alexander Shapiro.
\newblock Extremal problems on the set of nonnegative definite matrices.
\newblock \emph{Linear Algebra and its Applications}, 67:\penalty0 7--18, 1985.

\bibitem[Singer et~al.(2022)Singer, Polyak, Hayes, Yin, An, Zhang, Hu, Yang,
  Ashual, Gafni, et~al.]{singer2022make}
Uriel Singer, Adam Polyak, Thomas Hayes, Xi~Yin, Jie An, Songyang Zhang, Qiyuan
  Hu, Harry Yang, Oron Ashual, Oran Gafni, et~al.
\newblock Make-a-video: Text-to-video generation without text-video data.
\newblock \emph{arXiv preprint arXiv:2209.14792}, 2022.

\bibitem[Vandenberghe and Boyd(1996)]{vandenberghe1996semidefinite}
Lieven Vandenberghe and Stephen Boyd.
\newblock Semidefinite programming.
\newblock \emph{SIAM review}, 38\penalty0 (1):\penalty0 49--95, 1996.

\end{thebibliography}
